\documentclass{article}
\usepackage[T1]{fontenc}
\usepackage[utf8]{inputenc}
\usepackage[english]{babel}
\usepackage{amssymb}
\usepackage{amsmath}
\usepackage{amsfonts}
\usepackage{dsfont}
\usepackage{float}
\usepackage{graphicx}
\usepackage{wrapfig}
\usepackage{mathtools}
\usepackage{bbm}
\usepackage{amsthm}
\usepackage{ifthen}
\usepackage{graphicx}
\usepackage[hidelinks]{hyperref}
\usepackage{algorithm}
\usepackage{algpseudocode}
\usepackage{xcolor}
\usepackage{mathtools}
\usepackage{empheq}
\usepackage{mathrsfs}
\usepackage{multicol}
\usepackage{graphicx}
\usepackage{tcolorbox}
\usepackage{amsmath}
\usepackage{amsfonts}
\usepackage{tikz}
\usepackage{cite}
\usepackage{marvosym}
\usepackage{wasysym}
\usepackage{listings}
\usepackage{tensor}
\usepackage{diagbox}
\usepackage{authblk} 
\usepackage{xcolor} % to access the named colour LightGray
\usepackage{parskip}
\usepackage{subcaption}
\usepackage{xcolor,colortbl}

\definecolor{Gray}{gray}{0.85}
\definecolor{LightCyan}{rgb}{0.88,1,1}

\newcolumntype{a}{>{\columncolor{Gray}}c}
\definecolor{LightGray}{gray}{0.9}

\tcbuselibrary{breakable}

\newtheorem{theorem}{Theorem}[section]
\newtheorem{prop}[theorem]{Proposition}
\newtheorem{definition}[theorem]{Definition}
\newtheorem{remark}{Remark}[section]

\newtheorem{corollary}[theorem]{Corollary}

\newcommand{\ca}[1]{\mathcal{#1}}
\newcommand{\bb}[1]{\mathbb{#1}}

\newcommand{\p}{\mathbb{P}}

\newcommand{\set}[1]{\left\{#1\right\}}

\newcommand{\Rd}{\bb{R}^d}

\newcommand{\R}{\bb{R}}

\chardef\bslash=`\\ % p. 424, TeXbook
\numberwithin{equation}{section}

\newcommand{\N}{\mathbb{N}}

\newcommand{\T}{\mathbb{T}}
\newcommand{\Td}{\mathbb{T}^d}

\def\bm{\left( \begin{array}{cc}}
\def\endm{\end{array}\right)}

\newcommand{\be}{\begin{equation}}
\newcommand{\ee}{\end{equation}}
\newcommand{\ba}{\left(\begin{array}{c}}
\newcommand{\ea}{\end{array}\right)}
\newcommand{\bea}{\begin{eqnarray}}
\newcommand{\eea}{\end{eqnarray}}
\newcommand{\bee}{\begin{eqnarray*}}
\newcommand{\eee}{\end{eqnarray*}}
\newcommand{\ben}{\begin{enumerate}}
\newcommand{\een}{\end{enumerate}}

\usepackage[colorinlistoftodos,prependcaption,color=yellow,textsize=tiny,textwidth=2cm]{todonotes}

\usepackage[preprint,nonatbib]{neurips_2024}

\usepackage[utf8]{inputenc} % allow utf-8 input
\usepackage[T1]{fontenc}    % use 8-bit T1 fonts
\usepackage{hyperref}       % hyperlinks
\usepackage{url}            % simple URL typesetting
\usepackage{booktabs}       % professional-quality tables
\usepackage{amsfonts}       % blackboard math symbols
\usepackage{nicefrac}       % compact symbols for 1/2, etc.
\usepackage{microtype}      % microtypography
\usepackage{xcolor}         % colors

\title{THINNs: Thermodynamically Informed Neural Networks}

\author{
\begin{tabular}{cc}
    \begin{tabular}{c}
      Javier Castro\\
      Technische Universität Berlin, Fakultät II, 10623 Berlin, Germany\\
      \texttt{castro.medina@tu-berlin.de}
    \end{tabular}
    \begin{tabular}{c}
      Benjamin Gess\\
      Technische Universität Berlin, Fakultät II, 10623 Berlin, Germany,\\
      and Max Planck Institute for Mathematics in the Sciences 04103 Leipzig, Germany\\
      \texttt{benjamin.gess@mis.mpg.de}
    \end{tabular}
  \end{tabular}
}

\begin{document}
\maketitle

\begin{abstract}
\noindent Physics-Informed Neural Networks (PINNs) are a class of deep learning models aiming to approximate solutions of PDEs by training neural networks to minimize the residual of the equation. Focusing on non-equilibrium fluctuating systems, we propose a physically informed choice of penalization that is consistent with the underlying fluctuation structure, as characterized by a large deviations principle. This approach yields a novel formulation of PINNs in which the penalty term is chosen to penalize improbable deviations, rather than being selected heuristically. The resulting thermodynamically consistent extension of PINNs, termed THINNs, is subsequently analyzed by establishing analytical a posteriori estimates, and providing empirical comparisons to established penalization strategies.
\end{abstract}

\tableofcontents
\section{Introduction}

Physics-Informed Neural Networks (PINNs), as introduced in work \cite{raissi19}, building upon previous works like \cite{Dissanayake94,lagaris98}, are a class of deep learning models that integrate physical laws, expressed as partial differential equations (PDEs), directly into the training process. By embedding these governing equations into the loss function, PINNs can learn solutions that remain consistent with known physics, even with sparse or noisy data. As a consequence, a key aspect of PINNs is the choice of the loss function that penalizes the residual, that is, the deviation from solving the PDE. In this work, we provide a thermodynamic interpretation of this penalization, thereby offering a physics-informed criterion for its selection. 

More precisely, for a PDE of the form
\begin{align}\label{eq:gen-PDE-intro}
    \partial_t\rho - \ca{L}\rho=0,
\end{align}
on some domain $\Omega\subseteq \R^d$, with some integro-differential operator $\ca{L}$, and  initial condition $\rho_0$, PINNs seek to approximate the solution $\rho$ to \eqref{eq:gen-PDE-intro} by means of a neural network ansatz $\rho^\theta$, where $\theta \in \Theta$ denotes the collection of trainable parameters. The learning procedure is based on minimizing the defect of $\rho^\theta$ with respect to the governing equation, i.e.,
  $$ \underset{\theta\in\Theta}{\min}\,\ca{I}(\partial_t\rho^{\theta} - \ca{L}\rho^{\theta})$$
under a suitable choice of penalization functional $\mathcal{I}$. The selection of $\mathcal{I}$ is therefore of central importance to the method. 

In most existing approaches, however, this choice is made in an ad-hoc manner, typically by enforcing the residual in the $L^2$-norm. Together with the penalization of the initial condition, this leads to the optimization problem
\begin{align}\label{eq:pinn-loss}
   \underset{\theta\in\Theta}{\min} \int_0^T\|\partial_t\rho^{\theta} -  \ca{L}\rho^{\theta}\|_{L^2(\Omega)}^2 dt+\|\rho_0-\rho^{\theta}_0 \|_{L^2(\Omega)}^2.
\end{align}

The objective of the present work is to analyze the implicit physical meaning of the penalization in \eqref{eq:pinn-loss} in the context of nonequilibrium thermodynamical systems and their fluctuations. We demonstrate that in this context, this choice is not merely heuristic, but admits an interpretation grounded in thermodynamic principles. This insight naturally suggests a principled modification of the PINN framework, by choosing the penalization consistently with large deviation principles, which we term thermodynamically consistent physics-informed neural networks (THINNs).

We derive the concept of THINNs from three complementary perspectives (see Section~\ref{sec:derivation}). First, we show that THINNs arise naturally from geometric considerations, when interpreting PDEs as gradient flows on infinite-dimensional Riemannian manifolds. Second, we establish that THINNs are consistent with the large-deviation principles of nonequilibrium thermodynamics. Third, we derive THINNs as a canonical framework for balancing modeling and discretization errors in the simulation of fluctuating nonequilibrium thermodynamical systems. \\
From these considerations, we are led to a thermodynamically consistent choice of penalization of the residual. For instance, in the case of the symmetric simple exclusion process (SSEP) and the corresponding heat equation, the resulting optimization problem takes the form
\begin{align}\label{eq:intro-ssep}
    \underset{\theta\in\Theta}{\min}  \int_0^T\|\partial_t\rho^{\theta} -  \Delta\rho^{\theta}\|_{\dot H^{-1}_{\rho^{\theta}(1-\rho^{\theta})}}^2 + \mathcal{I}_0(\rho^{\theta}|\rho(0)),
\end{align}
where $\dot H^{-1}_{\rho^{\theta}(1-\rho^{\theta})}$ denotes the dual of the corresponding weighted Sobolev space, and $\mathcal{I}_0$ is the relative entropy, symmetrized with respect to $\{0,1\}$, see \eqref{eq:rel_entr_symm}.  

In Section~\ref{sec:THINNs}, we devise a numerical discretization of THINNs. This requires addressing the fact that neural networks do not, in general, satisfy periodic boundary conditions, and exploiting a variational representation of the loss functional \eqref{eq:intro-ssep} in order to render it computationally tractable.  

Rigorous approximation and a posteriori error estimates for THINNs are provided in Section~\ref{sec:analysis}, for the heat equation, the viscous Burgers equation, and the incompressible two-dimensional Navier--Stokes equations. These estimates are of \emph{a posteriori} type, in the sense that they bound the error between the exact PDE solution and the neural network approximation in terms of the residual loss functional. Compared to the classical PINN setting, several additional challenges arise: in particular, the residual is controlled only in weaker metrics, and initial fluctuations are penalized by nonlinear functionals such as the relative entropy. To address this, novel stability analysis to this context is introduced which is consistent with the THINN loss function, that is, it is based on estimating the relative entropy between the exact and approximate solutions.  

Finally, in Section~\ref{sec:numerical-results} we present numerical experiments for two model problems: the viscous Burgers equation in the small-viscosity regime, and the incompressible two-dimensional Navier--Stokes equations. In both cases, the proposed THINN approach is found to empirically clearly outperform classical PINNs.

\subsection{Review of the literature}
Neural networks have been intensively used to parametrize solutions of PDEs. In order to find appropriate parameters, the minimization of a loss function that leverages from a strong, weak or stochastic formulation of the equation has become a popular choice in the scientific community. An incomplete list of seminal papers in this direction is given by \cite{DGM, hure20, raissi19, E18, han18, li21}. 

\subsubsection*{Choice of the loss in PINNs }

One of the first works to question the suitability of $L^2$-penalization is \cite{E18}, where the authors introduce the Deep Ritz method, a deep learning-based approach for solving variational formulations of PDEs. An error analysis for this method with boundary penalties was later provided in \cite{mueller22}. Along similar lines, \cite{zang20} proposes weak adversarial networks (WANs), which exploit the weak formulation of PDEs and reformulate the learning task as the minimization of an operator norm through a min–max problem in the spirit of adversarial training. More recently, \cite{wang-chuwei22} introduces the notion of stability to evaluate the quality of solutions learned with small but nonzero physics-informed loss, focusing on Hamilton–Jacobi–Bellman equations from optimal control theory. In this context, the authors advocate replacing the $L^2$-norm with the $L^{\infty}$-norm to measure the physical loss, again leading to a min–max optimization problem addressed through adversarial techniques.

In \cite{bonito24}, the authors consider elliptic boundary value problems and exploit the fact that the error in the $H^{1}$-norm can be controlled by the residual measured in the $H^{-1}$-norm together with the boundary mismatch. Motivated by this estimate, they replace the standard $L^{2}$-based residual loss with this dual-norm formulation, leading to the consistent PINN (cPINN) method. Numerical experiments show that cPINNs yield improved accuracy compared to the classical $L^{2}$ approach.

It has been empirically observed that PINNs perform poorly when the PDE solution develops discontinuities, as in hyperbolic models where shocks form in finite time even from smooth initial data. In \cite{mishra-molinaro22}, the authors demonstrate this phenomenon for the Burgers’ equation, showing that the accuracy of PINNs deteriorates as the viscosity vanishes. To address this issue, \cite{ryck-mishra-molinaro24} introduce weak PINNs (wPINNs), which incorporate a loss function derived from weak formulations of conservation laws, consistent with the theory of entropy solutions. This approach yields promising numerical results and motivates a closer examination of residual-based methods for physical loss functions.

In the direction of improving wPINNs, \cite{chaumet-giesselman24} propose the approximation of weak entropy solutions to hyperbolic conservation laws by minimizing the $L_t^2 H^{-1}$-norm of the residual. In contrast to \cite{ryck-mishra-molinaro24}, where the dual norm is represented through a supremum over test functions, the approach of \cite{chaumet-giesselman24} follows the spirit of the Deep Ritz method \cite{E18}, maximizing the associated energy functional. Concretely, for a residual $R(\theta)$, they introduce an auxiliary potential $w$ satisfying $R(\theta) = (-\Delta) w$, so that $\|R(\theta)\|_{H^{-1}} = \|\nabla w\|_{L^2}$.

Deep neuronal networks have been introduced to the field of approximating gradient flows for PDEs in  \cite{park23}, where the authors propose a neural network formulation of the minimizing movement scheme of \cite{JKO98}.  A related time-discrete JKO-type scheme is also introduced in \cite{hu24} under the name energetic variational neural networks.

\subsubsection*{Model discovery and inference}

A different line of research has recently emerged that focuses on inferring thermodynamically consistent structures, such as the GENERIC formalism, directly from data. The viewpoint adopted in the present work is fundamentally distinct. Here, we assume that the thermodynamic structure of the microscopic system is known, and we exploit this knowledge to derive a physically consistent penalization of the residual in the loss functional. In contrast, the works discussed in this section aim to \emph{learn} the thermodynamic structure itself from data, subject to a prescribed structural framework.  

In \cite{zhang-shin-karniadakis22}, building on the GENERIC formalism \cite{oettinger05}, the authors introduce GENERIC-informed neural networks (GFINNs). The proposed framework seeks to infer, from data, the fundamental building blocks of GENERIC, such as the metric, entropy, and energy functionals. A more recent development \cite{huang24}, extending \cite{zhang-shin-karniadakis22,dietrich23}, introduces Statistical-Physics-Informed Neural Networks (Stat-PINNs). The focus of Stat-PINNs is to predict macroscopic dynamics from noisy microscopic simulations. In contrast, THINNs aim to produce directly a highly probable realization of the underlying microscopic system, without relying on trajectory data. 

\subsubsection*{PINNs in computational fluid dynamics}

PINNs have also emerged as a promising tool for solving fluid dynamics problems such as the Navier--Stokes equations. Since a comprehensive review of their applications in this context would go beyond of the scope it this article, we refer to \cite{Cai21}. In \cite{jin-cai-li-karniadakis21}, the authors introduce Navier--Stokes flow nets (NSFnets) to address a variety of fluid dynamics problems in two formulations: velocity--pressure and vorticity--velocity. Their study investigates the dependence of accuracy on the number of residual points and proposes a dynamic weight scheduling strategy that improves performance in both formulations, thereby demonstrating the potential of neural networks to approximate complex fluid flows.  

At the same time, limitations of PINNs have been reported. In \cite{chuang-22}, the authors document efficiency and expressivity issues when training PINNs on benchmark problems such as the two-dimensional Taylor vortex and cylinder flow. A broader analysis of the ability of PINNs to capture shocks and discontinuities in fluid dynamics is presented in \cite{Neelan24}. More recently, \cite{Wang24} propose an enhanced architecture incorporating adaptive weighting techniques, which yields improvements in both accuracy and training efficiency.  

We emphasize that, unlike the THINN framework introduced in the present work, all of these approaches retain the standard $L^2$ formulation of the loss functional and remain tied to deterministic PDE models.

\section{Derivation of THINNs}\label{sec:derivation}

The derivation and motivation of THINNs can be understood from several perspectives. We will first take a geometric viewpoint, relating to gradient flows, subsequently, a statistical physics viewpoint, relating to macroscopic fluctuations, and thirdly a scientific computing viewpoint, discussing numerical and modeling errors.

\textbf{Gradient flow perspective:} Many PDEs can be informally interpreted as gradient flows on infinite-dimensional Riemannian manifolds, for example, see \cite{JKO98} for diffusive equations, and \cite{oettinger05} for the general GENERIC framework for dissipative-convective PDEs. That is, informally, one can find a Riemannian manifold $M$ and an energy $\ca{E}\colon M\to\R$ so that the PDE can be written as a solution to the gradient flow 
\begin{align}\label{eq:gf-intro}
    \partial_t\rho=-\nabla_\rho \ca{E}(\rho),
\end{align}
where $\nabla_\rho$ denotes the Riemannian gradient on $M$ at $\rho$. Since 
  $$ \partial_t\rho^{\theta} + \nabla_{\rho^{\theta}} \ca{E}(\rho^{\theta}) \in  T_{\rho^{\theta}} M,$$   
the intrinsic choice of the penalization of the residual becomes 
\begin{align}\label{eq:gf-intro-min}
     \underset{\theta\in\Theta}{\min}\int_0^T \|\partial_t\rho^{\theta} + \nabla_{\rho^{\theta}} \ca{E}(\rho^{\theta})\|_{T_{\rho^{\theta}}M}^2.
\end{align}
Comparing to \eqref{eq:pinn-loss}, we thus deduce that the choice of the norm in the loss function corresponds to a \textit{choice} of Riemannian metric, and, thus, to the \textit{choice} of a gradient interpretation of the PDE.

To exemplify this, let us consider the heat equation
\begin{align}\label{eq:heat-intro}
    \partial_t\rho - \Delta\rho=0.
\end{align}
Indeed, this PDE admits many gradient flow formulations: For instance, let $M = L^2(\Omega)$, and $\ca{E}(\rho)=\frac{1}{2}\int |\nabla \rho|^2 dx$. Then \eqref{eq:heat-intro} takes the form \eqref{eq:gf-intro}, and the geometrically consistent PINN loss is identical to the classic PINN loss \eqref{eq:pinn-loss}.  

As a second instance, let $M$ be the space of probability measures $\mu = \rho dx$ with densities $\rho$, endowed with the  Wasserstein gradient flow structure \cite{adams-dirr-peletier-zimmer11, adams-dirr-peletier-zimmer13}, and set $\ca{E}(\rho)=\int\rho\log\rho dx$,
the Boltzmann entropy. Again, \eqref{eq:heat-intro} takes the form \eqref{eq:gf-intro}. With this gradient flow structure, the geometrically consistent dynamic PINN loss becomes
\begin{align}\label{eq:PINN_loss_new}
     \underset{\theta\in\Theta}{\min}\int_0^T \|\partial_t\rho^{\theta} - \Delta \rho^{\theta}\|_{\dot H^{-1}_\rho}^2.
\end{align}
In conclusion, the selection of the residual penalization is not a purely technical matter, but is intrinsically linked to the choice of a gradient flow structure for the underlying PDE. In particular, the classical choice of the $L^2$-penalization in \eqref{eq:pinn-loss} implicitly amounts to adopting a specific, and to some extent arbitrary, gradient flow interpretation. This observation motivates selecting the penalization by the physical principles inherent to the problem at hand, thereby motivating the introduction of thermodynamically consistent penalizations and the corresponding framework of THINNs.

Thereby, we are led to the question which gradient flow structure, and, thus, which loss to choose. This choice should be dictated by the physics of the underlying physical system. In the context of nonequilibrium thermodynamic systems and their fluctuations, the physically correct choice of a gradient flow structure is then determined via its large deviation principle of macroscopic fluctuations. For the principal relation between gradient flow structures and large deviations see \cite{adams-dirr-peletier-zimmer11,adams-dirr-peletier-zimmer13,peletier-redig-vafayi-14}. This links to the following second perspective on physically informed choices of the loss function via macroscopic fluctuation theory.

\textbf{Macroscopic fluctuation theory perspective: } In this context, the PDE \eqref{eq:gen-PDE-intro} is interpreted as describing the average, macroscopic behavior of a more complex underlying physical system. To make this more precise, consider a stochastic microscopic system of $N$ particles with empirical density field $\pi^N$. In the hydrodynamic limit $N \to \infty$, one has the convergence $\pi^N \to \rho$, where $\rho$ denotes the solution of \eqref{eq:heat-intro}.  Once this macroscopic behavior of the microscopic system $\pi^N$ is identified, one may study the fluctuations around the mean trajectory $\rho$ by means of large deviation principles. Roughly speaking, these principles quantify the probability of the microscopic system to follow a prescribed path of measures $(\bar\rho_t(x)\,dx)_{t\geq 0}$ in the form  
\begin{align}\label{eq:ldp-ssep}
    \mathbb{P}\big( \pi^N \approx \bar\rho \big) \;\sim\; \exp\!\left(-N^d \, \mathcal{I}(\bar\rho)\right),
\end{align}
where $\mathcal{I}$ denotes the so-called rate function. The functional $\mathcal{I}$ admits a decomposition into a static and a dynamical contribution,  
\[
    \mathcal{I} = \mathcal{I}_{\mathrm{dyn}} + \mathcal{I}_{\mathrm{stat}}.
\]  
In particular, $\mathcal{I}(\bar\rho) = 0$ if and only if $\bar\rho$ is a solution to the PDE.  

As a consequence, the large deviation principle \eqref{eq:ldp-ssep} provides direct guidance for the construction of thermodynamically consistent loss functions. Indeed, a natural choice in \eqref{eq:pinn-loss} is to replace the ad-hoc $L^2$ penalization by the rate functional $\mathcal{I}$ of the underlying microscopic system. Such a choice ensures that the training process penalizes nonphysical residuals in accordance with their physical improbability.  

For concreteness, and in order to highlight the relation with the gradient flow perspective, let us return to the heat equation \eqref{eq:heat-intro}. The heat equation arises as the hydrodynamic limit of a wide class of microscopic interacting particle systems. The associated large deviation rate functional, and therefore the corresponding thermodynamically consistent choice of loss function, depends on the specific microscopic dynamics under consideration.  For example, the large deviation principle (LDP) for the symmetric simple exclusion process (SSEP), see \cite{kipnis89,kipnis99}, identifies the dynamical part as
\begin{align}\label{eq:rate-ssep}
    \ca{I}_{\mathrm{dyn}}(\bar\rho)
    =\int_0^T \|\partial_t\bar\rho - \Delta \bar\rho\|_{\dot H^{-1}_{\bar\rho(1-\bar\rho)}}^2,
\end{align}
and the static part $\ca{I}_0$ as in \eqref{eq:intro-ssep}.

In addition, the form of the rate function \eqref{eq:rate-ssep} also identifies the  thermodynamically consistent gradient flow structure, see \cite{gess-heydecker25,adams-dirr-peletier-zimmer11,adams-dirr-peletier-zimmer13,peletier-redig-vafayi-14}. Precisely, the gradient flow structure for the SSEP is given by
\begin{align*}
    \ca{E}[\rho]=\int_{\Td}\rho\log\rho dx + \int_{\Td}(1-\rho)\log(1-\rho)dx,\text{ and }\ca{K}_{\rho}\xi=-\nabla\cdot(\rho(1-\rho)\nabla\xi).
\end{align*}

In conclusion, both from the perspective of macroscopic fluctuations, and from the geometric perspective, we are led to the thermodynamically/geometrically consistent choice of the loss function for PINNs of the form \eqref{eq:intro-ssep}.

\textbf{Scientific computing perspective: } In the context of scientific computing, the primary objective is to obtain reliable numerical simulations of real physical systems. PDEs of the form \eqref{eq:gen-PDE-intro} play the role of intermediate models: their discretization yields a computable object, which can then be employed to approximate the behavior of the underlying physical system. It is important to emphasize, however, that the PDE itself is not the ultimate target of the computation. Rather, it constitutes a surrogate model that facilitates the construction of efficient numerical schemes.

From this perspective, the overall accuracy of the simulation is influenced not only by the discretization error inherent in the numerical treatment of the PDE but also by the modeling error that arises from the approximation of the true physical system by the PDE model. Consequently, the total error naturally decomposes into two contributions: $\mathrm{total\ error \le modelling\ error+discretization\ error}$. A modeling agnostic Ansatz, like PINNs or finite elements, proceeds by minimizing the discretization error independently of the modeling error. It is at this point, that THINNs, in the context of nonequilibrium thermodynamics, instead directly minimize the total error, or, respectively, minimize the discretization error in a consistent way with the modeling error.

To make the above considerations more precise, let $\pi^N$ denote the stochastic physical system at finite scale, whose empirical average behavior is described in the hydrodynamic limit by the PDE \eqref{eq:gen-PDE-intro}.  In this framework, the \emph{modeling error} corresponds to the neglect of fluctuations of $\pi^N$ around the deterministic law \eqref{eq:gen-PDE-intro}.  Indeed, the large deviations principle allows to estimate the distance to an approximating function $\bar\rho$ via 
\begin{align}\label{eq:ldp}
	\p\big( \pi^N\approx \bar\rho \big) \sim e^{-N^d \ca{I}(\bar\rho)},
\end{align}
where $\ca{I}$ is the large deviations rate function. Therefore, in the context of fluctuating nonequilibrium thermodynamics, the size of the rate function can be understood as a measure for the total error. 

This motivates the development of approximation schemes for \eqref{eq:gen-PDE-intro} that are \emph{thermodynamically consistent}, in the sense that their numerical error is directly linked to the modeling error induced by neglecting fluctuations. The present framework of THINNs, based on a thermodynamically consistent choice of loss functional in PINNs, achieves precisely this: it aligns the numerical error with the physical error originating from the reduction of $\pi^N$ to $\bar\rho$. 

For the sake of exemplifying this, we consider again the heat equation. As above, this PDE arises as the hydrodynamic limit of the empirical density $\pi^N$ of the SSEP. The aim now is to find an approximation $\rho^M$ reaching a given level $\delta$ of error $d(\pi^N,\rho^M)$. Due to \eqref{eq:ldp-ssep}, that is, we want $\p\left(\| \pi^N -\rho^M \| \le \delta \right)$ to be large. The large deviation principle estimates 
\begin{align*}
    \p\left(\| \pi^N -\rho^M \| \le \delta \right) \approx \exp\left(-N^d\inf_{\bar\rho\in B_{\delta}(\rho^M)} \Big( \int_0^T\|\partial_t\bar\rho -  \Delta\bar\rho\|_{\dot H^{-1}_{\bar\rho(1-\bar\rho)}}^2\,dt + \mathcal{I}_0(\bar\rho|\rho(0))\Big)
    \right).
\end{align*}
As a consequence, the size of the rate function $\ca{I}(\rho^M)$ is a good measure of error since it controls the probability of deviation from the true, underlying system. This motivates, firstly, to choose the rate function as one of the measures of error in the simulations provided below, and, secondly, it offers yet another perspective leading to the THINN choice of loss function, and the corresponding optimization problem.

\section{Discretization of THINNs}\label{sec:THINNs}

In this section we derive the discretization of THINNs. The discretization involves several steps, each of which presents distinct challenges.  

Firstly, the ansatz space of neural network functions does not, in general, exactly satisfy the (periodic) boundary conditions of the underlying PDE. From the perspective of large deviation theory, this implies that the associated rate functional would assign infinite cost to such functions. Accordingly, the rate functional must be suitably modified in order to accommodate the neural network ansatz.  

Secondly, the discretization of the thermodynamically consistent loss functional presents significant challenges. Analytically, this loss can be represented in several equivalent forms: (i) as a dual norm in a weighted Sobolev space, (ii) as a supremum over controls, as in \cite{kipnis89} (cf.\ \eqref{eq:rate-sup-form}), or (iii) as an infimum over solutions of an elliptic PDE, commonly referred to as the skeleton equation (cf.\ \eqref{eq:rate-ssep}). The dual-norm formulation requires normalization by the test-function norm, which is prone to numerical instabilities. The supremum-based formulation leads to a min–max problem, which could in principle be addressed via adversarial-type training, however, as for GANs, the trainability becomes highly challenging. In practice, as supported by the empirical studies of this work, the skeleton-equation formulation proves to be the most feasible. In particular, by introducing a modified neural network architecture, automatic differentiation can be exploited to integrate the elliptic PDE and construct candidate solutions, thereby circumventing the need for a nested optimization procedure.  

Thirdly, the approximation of these quantities necessarily involves the selection of a suitable quadrature rule to replace continuous integrals by computable sums.

The remainder of this section is organized as follows. In Section~\ref{sec:numerical-implementation}, we address the discretization of the static and dynamic large deviation principles: we reformulate the dual norm of the residual by introducing an auxiliary divergence-free model, which renders the computation of the rate functional feasible. Furthermore, we define the admissible sets of neural networks, specify the notion of quadrature in this context, and derive the discretized form of the loss. Finally, we summarize these components into a complete algorithmic formulation of THINNs.

\subsection{Numerical implementation}\label{sec:numerical-implementation}

As a first step, we propose a numerical surrogate for the rate functional in \eqref{eq:intro-ssep} that admits non-periodic functions as valid inputs while remaining implementable in practice. The difficulty is that if the parametrized model is non-periodic, then it falls outside the natural domain of the rate functional, which would otherwise assign it infinite value. To overcome this, we revisit the weak formulation of the skeleton equation and explicitly account for the resulting boundary contributions.  

We work in the abstract setting of Appendix~\ref{sec:GF-IPS}. In particular, let $\pi^N$ denote a physical system satisfying a large deviation principle with rate functional of the form \eqref{eq:GF-IPS-rate-2}. For a fixed path $\bar\rho \in D([0,T],M)$ and admissible control $g \in L^2([0,T]\times\mathbb{T}^d)$, the skeleton equation is given by
\begin{align}\label{eq:skeleton}
    \partial_t \bar\rho_t
    + \nabla \cdot \big(\Phi(\bar\rho)\nabla f[\bar\rho]\big)
    = -\nabla \cdot \big(\Phi^{1/2}(\bar\rho) g\big).
\end{align}

Now suppose that $\rho^{\theta}$ denotes a parametrized model which is not necessarily periodic. The weak formulation of \eqref{eq:skeleton} then reads
\begin{align}\label{eq:weak-formulation-with-boundary}
    \int_0^T \int_{\mathbb{T}^d}
        \Big( \partial_t \rho^{\theta}
        + \nabla \cdot (\Phi(\rho^{\theta})\nabla f[\rho^{\theta}]) \Big)\varphi \, dx \, dt
    &= \int_0^T \int_{\mathbb{T}^d} 
        \Phi(\rho^{\theta})^{1/2} g \cdot \nabla \varphi \, dx \, dt \nonumber\\
    &\quad - \int_0^T \int_{\partial Q_d} 
        \Phi(\rho^{\theta})^{1/2} g \, \varphi \cdot d\vec{S}(x) \, dt,
\end{align}
for all smooth test functions $\varphi \in C^{\infty}([0,T]\times Q_d)$. Here $Q_d=[0,1]^d$ denotes the unit cube, and $d\vec{S}(x)=\vec{n}(x)\,dS(x)$ is the surface measure on $\partial Q_d$ with outward unit normal vector $\vec{n}$.

If $\rho^{\theta}$ is not periodic, then \eqref{eq:weak-formulation-with-boundary} remains well-defined, provided that the boundary value of the flux $\Phi(\rho^{\theta})^{1/2} g \cdot \vec{n}$ on $\partial Q_d$ is specified. This corresponds to imposing a Neumann-type boundary condition, and motivates the following modified definition of the rate functional, which incorporates boundary contributions:
\begin{align}\label{eq:dyn-rate-numerical-0}
    I^G_{\mathrm{dyn}}(\theta)
    = \inf\Bigg\{\int_{Q_d} |g|^2 \, dx \;:\;
    \begin{cases}
        \partial_t \rho^{\theta}
        + \nabla \cdot \big(\Phi(\rho^{\theta}) \nabla f[\rho^{\theta}]\big)
        = -\nabla \cdot \big(\Phi(\rho^{\theta})^{1/2} g\big), \\[0.4em]
        \big(\Phi(\rho^{\theta})^{1/2} g\big)\cdot \vec{n}(t,x) = G(t,x),
        \quad (t,x)\in [0,T]\times \partial Q_d,
    \end{cases}
    \Bigg\},
\end{align}
where $G \colon [0,T]\times \partial Q_d \to \mathbb{R}$ is a prescribed boundary flux.  

The corresponding Neumann skeleton equation is understood in a weak sense: a function $g \in L^2([0,T]\times Q_d)$ is a weak solution of \eqref{eq:dyn-rate-numerical-0} if, for all test functions $\varphi \in C^{\infty}([0,T]\times Q_d)$,  
\begin{align}\label{eq:weak-formulation-with-boundary-2}
    \int_0^T \int_{\mathbb{T}^d} 
        \Big( \partial_t \rho^{\theta}
        + \nabla \cdot (\Phi(\rho^{\theta}) \nabla f[\rho^{\theta}]) \Big)\varphi \, dx \, dt
    &= \int_0^T \int_{\mathbb{T}^d} 
        \Phi(\rho^{\theta})^{1/2} g \cdot \nabla \varphi \, dx \, dt \nonumber\\
    &\quad - \int_0^T \int_{\partial Q_d} G \varphi \, dS(x)\,dt,
\end{align}
where $dS(x)$ denotes the surface measure on $\partial Q_d$.

Two issues remain to be resolved: the specification of $G$ in \eqref{eq:dyn-rate-numerical-0}, and the fact that for a given $\rho^{\theta}$ one still needs to solve an elliptic problem. To facilitate the solution of the skeleton equation, we introduce the function class
\begin{align}\label{eq:W-def}
    \mathcal{W} = \Big\{ w^{\theta} \colon [0,T]\times Q_d \to \mathbb{R}^d \;\Big|\; 
        \theta \in \Theta,\; w^{\theta} \in C^{1,2}([0,T]\times Q_d) \Big\},
\end{align}
where $\Theta \subset \mathbb{R}^p$ is a parameter set. The class $\mathcal{W}$ can be chosen to coincide with fully-connected neural networks.  

We now define the ansatz space for the density by
\begin{align}\label{eq:M_def}
    \mathcal{M} = \Big\{ \rho^{\theta} = -\nabla \cdot w^{\theta} \;:\; w^{\theta} \in \mathcal{W} \Big\}.
\end{align}
For any $\rho^{\theta} \in \mathcal{M}$, the left-hand side of \eqref{eq:skeleton} can be rewritten as
\[
    \partial_t \rho^{\theta}
    + \nabla \cdot \big(\Phi(\rho^{\theta}) \nabla f[\rho^{\theta}]\big)
    = -\nabla \cdot \Big( \partial_t w^{\theta}
        + \Phi(\rho^{\theta}) \nabla f[\rho^{\theta}] \Big).
\]
It is convenient to introduce the auxiliary function $r^{\theta} \colon [0,T]\times Q_d \to \mathbb{R}^d$ defined by
\begin{align}\label{eq:r-theta}
    r^{\theta} = \partial_t w^{\theta} 
        + \Phi(\rho^{\theta}) \nabla f[\rho^{\theta}].
\end{align}
With this notation, the elliptic problem for $g$ reduces to
\[
    -\nabla \cdot r^{\theta} = -\nabla \cdot \big(\Phi(\rho^{\theta})^{1/2} g\big).
\]
Under sufficient regularity of the network models, one can deduce that
\begin{align}\label{eq:g-numerical}
    \bar{g}
    = \Big( \partial_t w^{\theta} 
        + \Phi(\rho^{\theta}) \nabla f[\rho^{\theta}] \Big)
      \, \Phi(\rho^{\theta})^{-1/2},
\end{align}
is a weak solution of the Neumann skeleton equation with boundary condition $G = r^{\theta}\cdot \vec{n}$, in the sense of \eqref{eq:weak-formulation-with-boundary-2}. This motivates the following redefinition of the dynamical rate functional $I_{\text{dyn}} \colon \Theta \to \mathbb{R}_+$:
\begin{align}\label{eq:dyn-rate-numerical}
    I_{\text{dyn}}(\theta)
    = \inf \bigg\{
        \int_{Q_d} |g|^2 dx
        \;\colon\;
        \begin{aligned}
            \partial_t \rho^{\theta}
            + \nabla\cdot\big(\Phi(\rho^{\theta}) \nabla f[\rho^{\theta}]\big)
            &= -\nabla \cdot \big(\Phi(\rho^{\theta})^{1/2} g\big), \\[0.25em]
            \big(\Phi(\rho^{\theta})^{1/2} g\big)\cdot \vec{n}(x)
            &= r^{\theta}\cdot \vec{n}(x),\quad x \in \partial Q_d,
        \end{aligned}
    \bigg\},
\end{align}
where 
\[
    r^{\theta} = \partial_t w^{\theta}
        + \Phi(\rho^{\theta}) \nabla f[\rho^{\theta}].
\]
The admissible set in \eqref{eq:dyn-rate-numerical} is nonempty, since $\bar{g}$ defined in \eqref{eq:g-numerical} is always a valid candidate. Consequently,
\begin{align}\label{eq:I-finite}
    0 \le I_{\text{dyn}}(\theta)
    \le \int_0^T \int_{\mathbb{T}^d} |\bar{g}|^2 \, dx \, dt
    <\infty.
\end{align}

The boundary term in the weak formulation \eqref{eq:weak-formulation-with-boundary} can be interpreted as a quantitative measure of the deviation of the approximation from periodicity. In the definition of \eqref{eq:dyn-rate-numerical}, this discrepancy is attributed to the vector field $r^{\theta}$, in the sense that
\begin{align*}
    \int_0^T \int_{\partial Q_d} r^{\theta}\cdot \vec{n} \varphi \, ds(x)\, dt
\end{align*}
vanishes whenever both $r^{\theta}$ and the test function $\varphi$ are periodic. Since the periodicity of $r^{\theta}$ will be essential in the derivation of the a posteriori estimates, this boundary contribution will be explicitly incorporated into the loss functional; see equation \eqref{eq:bc} below.

To complete the definition of the loss functional, it remains to incorporate the initial and boundary conditions. Let $\rho_0 \in L^2(\mathbb{T}^d)$ denote the initial datum. Denote by $\{C_i\}_{i=1}^d$ the facets of $Q_d=[0,1]^d$ such that $0 \in C_i$, and let $\{e_i\}_{i=1}^d$ be the canonical basis of $\mathbb{R}^d$, ordered so that $e_i$ is the inward normal of $C_i$. In particular, $C_i+e_i$ is the face opposite to $C_i$ on the cube.  

The deviation from the initial profile is quantified via the large deviation functional
\begin{align}\label{eq:function-h}
    \mathcal{I}_0(\rho(0)) = \int h(\rho(0,x),\rho_0(x)) \, dx,
\end{align}
for a measurable function $h \colon [0,1]\times[0,1]\to \mathbb{R}$, specified according to the physical system under consideration.  

The periodic boundary conditions are enforced in a soft form. Specifically, we introduce the losses
\begin{align}
    \label{eq:ic}
    I_0(\theta) &= \int h(\rho^{\theta}(0,x),\rho_0(x)) \, dx,\\
    \label{eq:bc}
    \mathcal{L}_{\mathrm{bc}}(\theta) &= \ell_T(\rho^{\theta}) + \ell_T(\nabla\rho^{\theta}) + \ell_T(r^{\theta}),
\end{align}
where for a measurable function $\varphi \colon [0,1]^d \to \mathbb{R}^n$, with $n \in \{1,d\}$, we set
\[
    \ell_T(\varphi) = \int_0^T \sum_{i=1}^d \int_{C_i} \big|\varphi(x+e_i)-\varphi(x)\big|^2 \, dS(x)\, dt.
\]

With these ingredients, the full loss functional $\mathcal{L} \colon \Theta \to [0,\infty)$ is defined by
\begin{align}\label{eq:full-loss}
    \mathcal{L}(\theta) = I_0(\theta) + I_{\mathrm{dyn}}(\theta) + \mathcal{L}_{\mathrm{bc}}(\theta).
\end{align}
The numerical method then seeks
\begin{align}\label{eq:main-min}
    \theta^* \in \underset{\theta \in \Theta}{\arg\min}\,\mathcal{L}(\theta).
\end{align}

\begin{definition}\label{def:quadrature}
    Let $[0,T]\times [0,1]^d$ denote the time--space domain, and let $\{C_i\}_{i=1}^d$ denote the facets of $[0,1]^d$ such that $0 \in C_i$ for each $i=1,\dots,d$.  
    A \emph{quadrature scheme} $(Q,W)$ for \eqref{eq:full-loss} consists of a collection of collocation data sets
    \[
        Q = \bigl\{ Q_{\mathrm{phy}}, Q_{\mathrm{ic}}, Q_{\mathrm{bd}} \bigr\},
    \]
    together with associated positive weights
    \[
        W = \bigl\{ \Delta_t, \Delta_x^{\mathrm{phy}}, \Delta_x^{\mathrm{ic}}, \Delta_x^{\mathrm{bd}} \bigr\},
    \]
    satisfying
    \[
        Q_{\mathrm{phy}} = Q^t_{\mathrm{phy}} \times Q^x_{\mathrm{phy}} \subset [0,T]\times[0,1]^d,\quad 
        Q_{\mathrm{ic}} \subset \set{t=0}\times [0,1]^d,\quad 
        Q_{\mathrm{bd}} = \bigcup_{i=1}^d Q_{\mathrm{bd}}^i,
    \]
    where $Q_{\mathrm{bd}}^i \subset C_i$ for each $i=1,\dots,d$.
\end{definition}

\begin{remark}
    In the present work we employ Monte Carlo quadrature to approximate the integrals in \eqref{eq:full-loss}, for which Definition~\ref{def:quadrature} is sufficient. The framework, however, can be adapted to accommodate more sophisticated quadrature rules without further conceptual changes. 
\end{remark}

To make the implementation of \eqref{eq:full-loss} precise, we now recall the general network spaces.  
For $k,l\in\mathbb{N}$, let $\mathscr{N}(\mathbb{R}^k,\mathbb{R}^l)$ denote the class of fully connected feedforward neural networks mapping $\mathbb{R}^k$ to $\mathbb{R}^l$, equipped with smooth activation functions. Since the subsequent analysis does not depend on the precise number of layers or neurons per layer, we only specify input and output dimensions. The concrete architectures employed in the numerical experiments will be described in the corresponding section.  

\begin{definition}\label{def:networks-def}
    Let $d \in \mathbb{N}_{\ge 1}$ denote the spatial dimension, and let $T>0$.  
    We define the class of admissible networks
    \begin{align}\label{eq:w-nets}
        \mathcal{W} \coloneqq \left\{ \rho\big|_{[0,T]\times[0,1]^d} \;\colon\; \rho \in \mathscr{N}(\mathbb{R}^{d+1},\mathbb{R}^d) \right\}.
    \end{align}
\end{definition}

The class $\mathcal{W}$ is consistent with the restrictions imposed in \eqref{eq:W-def}.  
The associated ansatz space $\mathcal{M}$ is then defined as in \eqref{eq:M_def}.  
A summary of the full procedure is provided in Algorithm~\ref{alg:main-algorithm}.

\begin{algorithm}
	\caption{THINNs Algorithm} 
    \label{alg:main-algorithm}
	\begin{algorithmic}[1]
		\State \textbf{Input}: initial datum $\rho_0$, network models $w^{\theta}\in\ca{W}$ as in Definition \ref{def:networks-def}, initial parameter $\theta_0\in\Theta$, total steps $S\in\N$, learning rate $\eta\in(0,1)$, and a quadrature scheme $(Q,W)$ as in Definition \ref{def:quadrature}.\\
        
		\State \textbf{Output}: Neural network parameter $\theta^{*}\in\Theta$.
        
		\For {$i\in\{0,...,S-1\}$}
            \State Fix $\theta=\theta_i$ 
            \State Compute the discretized loss functions as
\begin{align}
    \label{eq:quadrature-phy}
	\hat{I}_{\text{dyn}}(\theta) &= \sum_{(t,x)\in Q_{\text{phy}}}\left|\partial_{t}w^{\theta}(t,x)+\Phi(\rho^{\theta}(t,x))\nabla f[\rho^{\theta}(t,x)]\right|^2\Phi(\rho^{\theta}(t,x))^{-1}\Delta_t\Delta^{\text{phy}}_x,\\
    \label{eq:quadrature-ic}
    \hat{I}_{0}(\theta) &= \sum_{(t,x)\in Q_{\text{ic}}} h(\rho^{\theta}(t,x),\rho_0(x)) \Delta^{\text{ic}}_x,\\
    \label{eq:quadrature-bc}
	\hat{\ca{L}}_{\text{bc}}(\theta) &= \sum_{i=1}^d\sum_{(t,x)\in Q^i_{\text{bc}}} |\rho^{\theta}(t,x) - \rho^{\theta}(t,x+e_i)|^2 \Delta_t\Delta^{\text{bc}}_x.
\end{align}
            \State Update
            \begin{align*}
                \theta_{i+1}=\theta_i-\eta\nabla_{\theta}\left(\hat{I}_0(\theta_i)+\hat{I}_{\text{dyn}}(\theta_i)+\hat{\ca{L}}_{\text{bc}}(\theta_i)\right).
            \end{align*}
		\EndFor
		\State Return $\theta^*=\theta^N$.
	\end{algorithmic} 
\end{algorithm}

{
\subsection{Numerical implementation for the Navier--Stokes equations}\label{sec:numerical-implementation-nse}

The framework discussed above applies to underlying physical systems corresponding to the heat and Burgers’ equations. The case of the Navier--Stokes equations is, in fact, somewhat simpler. The incompressible Navier--Stokes system is written as  
\begin{align}\label{eq:nse-numerical-implementation}
    \partial_t u - \nu \Delta u + \nabla\cdot(u \otimes u) + \nabla p = 0, \quad
    \nabla\cdot u = 0,
\end{align}
with initial condition $u(0) = u_0$, where $u \colon [0,T]\times\mathbb{T}^d \to \mathbb{R}^d$ denotes the velocity field, $p \colon [0,T]\times\mathbb{T}^d \to \mathbb{R}$ is the pressure, and $u_0$ is a prescribed initial condition.   

In this context, the dynamical rate functional is given by the $L^2_TH^{-1}$ norm, which can be computed efficiently by means of Fourier analysis, while the static rate reduces to the $L^2$-distance to the initial profile. For a detailed discussion of the underlying physical system and the derivation of the associated rate functional, we refer to \cite{quastel98,gess-heydecker-wu24}. In particular, in this setting, the thermodynamically consistent geometry is flat.

For a real-valued function $u \in L^2(\mathbb{T}^d)$, we denote its Fourier coefficients by  
\begin{align*}
    \hat{u}_k = \frac{1}{(2\pi)^d} \int_{\mathbb{T}^d} u(x)\, e^{-ik\cdot x}\, dx,
    \qquad k \in \mathbb{Z}^d.
\end{align*}
Standard Fourier analysis yields the following representation of the $H^{-1}$-norm:
\begin{align*}
    \|u\|_{H^-1}^2=\sum_{k\in\bb{Z}^d,k\neq0}\frac{|\hat{u}_k|^2}{|k|^2},
\end{align*}
where $|k|^2 = k_1^2 + \cdots + k_d^2$.  Consequently, for a time-dependent distribution $u \in L^2_TH^{-1}$, its norm is expressed as  
\begin{align*}
    \int_0^T\sum_{k\in\bb{Z}^d,k\neq0}\frac{|\hat{u}_k(t)|^2}{|k|^2}dt.
\end{align*}

We now introduce the parametrization of the ansatz space by means of neural networks, specializing to the two-dimensional case ($d=2$). Let $\varphi^{\theta}\colon [0,T]\times\mathbb{T}^2 \to \mathbb{R}$ denote a neural network parametrized by $\theta$. We define the associated velocity field $u^{\theta}$ via the stream function formulation
\begin{align}\label{eq:velocity-nse}
    u^{\theta}(t,x) =
    \begin{pmatrix}
        \partial_y \varphi^{\theta}(t,x) \\[0.5em]
        -\partial_x \varphi^{\theta}(t,x)
    \end{pmatrix},
\end{align}
which is divergence-free by construction. This substantially reduces the complexity of the optimization problem, since the incompressibility constraint in \eqref{eq:nse-numerical-implementation} is automatically satisfied.  

The resulting model therefore consists of two scalar neural networks: $\varphi^{\theta}$, generating the divergence-free velocity field through \eqref{eq:velocity-nse}, and $p^{\theta}$, representing the pressure.

Finally, we define the residual of the Navier--Stokes system \eqref{eq:nse-numerical-implementation} as  
\begin{align*}
    R^{\theta} = \partial_t u^{\theta}
        - \nu \Delta u^{\theta}
        + \nabla\cdot(u^{\theta}\otimes u^{\theta})
        + \nabla p^{\theta}.
\end{align*}
We approximate the associated dynamical rate functional by
\begin{align}\label{eq:nse-rate}
    I_{\mathrm{dyn}}^{\text{NS}}(\theta)
    = \int_0^T\sum_{k\neq0}\frac{|R^{\theta}_k(t)|^2}{|k|^2}dt.
\end{align} 
Since the static rate functional is given by the $L^2$-distance to the initial profile, the function $h$ in \eqref{eq:ic} is chosen as $h(x,y) = |x-y|^2$. To be more precise, the losses are given by
\begin{align}
    \label{eq:ic-ns}
	I^{\text{NS}}_{0}({\theta}) &= \int |u^{\theta}(0,x)-u_0(x)|^2dx,\text{ and}\\
	\label{eq:bc-n}
    \ca{L}^{\text{NS}}_{\text{bc}}({\theta}) &= \ell_T(u^{\theta}) +\ell_T(\nabla u^{\theta}).
\end{align}
With these ingredients we now introduce the full loss functional for the Navier-Stokes equation $\ca{L}\colon\Theta\to[0,\infty)$ which is given by
\begin{align}\label{eq:full-loss-nse}
    \ca{L}^{\text{NS}}(\theta) = I^{\text{NS}}_0(\theta) + I^{\text{NS}}_{\text{dyn}}(\theta) +\ca{L}^{\text{NS}}_{\text{bc}}(\theta).
\end{align}
Finally, we remark that the quadrature scheme for the Navier--Stokes setting differs slightly from that introduced in Definition~\ref{def:quadrature}. The essential distinction is that the collocation points for the dynamical rate functional involve only the temporal component, while the spatial discretization is encoded by truncating the Fourier representation of the $H^{-1}$-norm.

\begin{definition}\label{def:quadrature-nse}
    Let $[0,T]\times [0,1]^d$ denote the time--space domain, and let $\{C_i\}_{i=1}^d$ denote the faces of $[0,1]^d$ such that $0\in C_i$ for each $i=1,\dots,d$.  
    A \emph{quadrature scheme} $(Q,W,N)$ for \eqref{eq:full-loss-nse} consists of a collection of collocation sets
    \[
        Q = \{ Q_{\mathrm{phy}}, Q_{\mathrm{ic}}, Q_{\mathrm{bd}} \},
    \]
    together with positive weights
    \[
        W = \{ \Delta_t, \Delta_x^{\mathrm{ic}}, \Delta_x^{\mathrm{bd}} \},
    \]
    and a positive integer $N$, subject to
    \[
        Q_{\mathrm{phy}} \subset [0,T], \qquad 
        Q_{\mathrm{ic}} \subset \{0\}\times[0,1]^d, \qquad 
        Q_{\mathrm{bd}} = \bigcup_{i=1}^d Q_{\mathrm{bd}}^i, 
    \]
    where $Q_{\mathrm{bd}}^i \subset C_i$ for all $i=1,\dots,d$.
\end{definition}

Algorithm~\ref{alg:alg-nse} summarizes the main steps of the proposed method in the Navier--Stokes setting. Recall that for $k,l\in\mathbb{N}$, we denote by $\mathscr{N}(\mathbb{R}^k,\mathbb{R}^l)$ the class of fully connected neural networks mapping $\mathbb{R}^k$ to $\mathbb{R}^l$, equipped with smooth activation functions. In the present case, the ansatz space is given by
\begin{align}\label{eq:w-nets-nse}
    \mathcal{W}^{\mathrm{NS}} = \mathscr{N}(\mathbb{R}^3,\mathbb{R}^2).    
\end{align}
\begin{algorithm}
	\caption{THINNs Algorithm: Two-dimensional Navier-Stokes equation} 
    \label{alg:alg-nse}
	\begin{algorithmic}[1]
		\State \textbf{Input}: initial datum $\rho_0$, network models $(\varphi^{\theta},p^{\theta})\in\ca{W}^{\text{NS}}$, initial parameter $\theta_0\in\Theta$, total steps $S\in\N$, learning rate $\eta\in(0,1)$, and a quadrature scheme $(Q,W,N)$ as in Definition \ref{def:quadrature-nse}.\\
        
		\State \textbf{Output}: Neural network parameter $\theta^{*}\in\Theta$.
        
		\For {$i\in\{0,...,S-1\}$}
            \State Fix $\theta=\theta_i$. 
            \State Compute $u^{\theta}$ as in \eqref{eq:velocity-nse} with stream function $\varphi^{\theta}$.
            \State Compute the discretized loss functions as
\begin{align}
    \label{eq:quadrature-phy-nse}
	\hat{I}^{\text{NS}}_{\text{dyn}}(\theta) &= \sum_{t\in Q_{\text{phy}}}\sum_{|k|_{\infty}\le N,k\neq0}\frac{|R^{\theta}_k(t)|^2}{|k|^2}\Delta_t,\\
    \label{eq:quadrature-ic-nse}
    \hat{I}^{\text{NS}}_{0}(\theta) &= \sum_{(t,x)\in Q_{\text{ic}}} h(u^{\theta}(t,x),u_0(x)) \Delta^{\text{ic}}_x,\\
    \label{eq:quadrature-bc-nse}
	\hat{\ca{L}}^{\text{NS}}_{\text{bc}}(\theta) &= \sum_{i=1}^d\sum_{(t,x)\in Q^i_{\text{bc}}} |u^{\theta}(t,x) - u^{\theta}(t,x+e_i)|^2 \Delta_t\Delta^{\text{bc}}_x.
\end{align}
            \State Update
            \begin{align*}
                \theta_{i+1}=\theta_i-\eta\nabla_{\theta}\left(\hat{I}^{\text{NS}}_0(\theta_i)+\hat{I}^{\text{NS}}_{\text{dyn}}(\theta_i)+\hat{\ca{L}}^{\text{NS}}_{\text{bc}}(\theta_i)\right).
            \end{align*}
		\EndFor
		\State Return $\theta^*=\theta^N$.
	\end{algorithmic} 
\end{algorithm}

\section{Approximatoin error and a-posteriori analysis}\label{sec:analysis}

In this section we establish convergence estimates for THINNs applied to three prototypical equations: the heat equation, the viscous Burgers equation, and the incompressible two-dimensional Navier--Stokes equations. The estimates obtained are of \emph{a posteriori} type, in the sense that they provide bounds on the error between the exact solution of the PDE and the neural network approximation in terms of the residual loss functional. Thus, whenever the residual loss can be trained to be small, or even converge to zero, one obtains quantitative control of the approximation error.  

We analyze two different settings. In the first, the approximation error is measured within a finite-dimensional ansatz space consisting of periodic functions, so that the residual loss can be directly related to the error in the periodic framework. In the second, we address the practical difficulty that neural network ansatz functions are not guaranteed to be periodic. In this case, the error estimates must be formulated for the actual parametrization defined in \eqref{eq:W-def} and \eqref{eq:w-nets-nse}, and additional boundary contributions naturally appear in the analysis.

It should be emphasized that the estimates derived in this section do not take into account the further errors arising from spatial discretization. In particular, the quadrature errors introduced when approximating the weighted Sobolev norms remain to be incorporated in the fully discrete error analysis, which will be treated separately.

The derivation of {a posteriori} estimates in the present framework entails several nontrivial difficulties. First, the THINN loss functional is formulated in terms of a negative-order Sobolev norm, so that the residual is controlled only in a weak, non-flat metric. This stands in contrast to the classical PINN setting, where the residual is measured in $L^2$ and standard {a posteriori} PDE estimates can be invoked directly. Second, the treatment of the initial discrepancy is likewise more involved, since it is penalized in a nonlinear manner, for instance via the relative entropy. Consequently, the stability analysis must be developed in a manner consistent with these thermodynamic structures. In this work, this is achieved by basing the stability analysis on the evolution of the relative entropy between solutions. 

\subsection{Heat equation and SSEP}
Let $\nu>0$ and consider the parabolic equation with periodic boundary conditions,
\begin{align}\label{eq:heat}
	\partial_t \rho &= \nu \Delta \rho,\quad  \rho(0)=\rho_0,
\end{align}
with $\rho_0$ an initial datum. We understand \eqref{eq:heat} as arising as the hydrodynamic limit of the SSEP, thus describing its mean behavior. In the forthcoming analysis, it will become apparent that we require at least $\rho\in L^2([0,T],H^1(\Td))$. By standard parabolic regularity theory, it will be sufficient to take $\rho_0\in L^2(\Td)$.

The estimate bounds the distance, in relative entropy, between functions $\rho^\theta$ in a set of periodic functions, and the true solution to the heat equation in terms of the rate function $\ca{I}(\rho^\theta)$. The definition of the domain of the rate function $D([0,T],M)$, is given in the Appendix section \ref{sec:GF-IPS}.

\begin{theorem}\label{thm:heat-analytical-estimate}
Consider a parametrized set of periodic functions
\begin{align}\label{eq:heat-analytical-setN}
    \ca{N}=\left\{\rho^{\theta}\colon[0,T]\times\Td\to(0,1)\ \big|\ \rho^{\theta}\in D([0,T],M),\text{ and }\theta\in\Theta\right\},
\end{align}
where $\Theta\subset\R^p$, $p\in\N$. Let $\rho_0\in L^2(\Td)$ take values in $[\delta,(1-\delta)]$, and let $\rho^{\theta}\in\ca{N}$ such that
\begin{align}\label{eq:log-rho-theta-as}
    \log\rho^{\theta}\in L^2([0,T],H^1(\Td))\text{ and }\ca{F}(\rho^{\theta}(\cdot))\in L^2([0,T]).   
\end{align}
Further $\rho\in L^2([0,T],H^1(\Td))$ be the weak solution of \eqref{eq:heat} and assume
\begin{align}\label{eq:log-rho-as}
    \ca{F}(\rho(\cdot))\in L^2([0,T]).    
\end{align}
Then 
\begin{align*}
        \ca{H}^s(\rho^{\theta}(T)|\rho(T))&+\frac{\nu}{2}\int_0^T\ca{F}^s(\rho^{\theta}(t)|\rho(t))dt\le\ca{I}_0(\rho^{\theta})+\nu^{-1}\ca{I}_{\text{dyn}}(\rho^{\theta}),
\end{align*}
where $\ca{H}^s$ and $\ca{F}^s$ are given in \eqref{eq:symmetric-ent} and \eqref{eq:symmetric-fish}
\end{theorem}
\begin{proof}
    Let $\rho^{\theta}\in\ca{N}$ be such that $\ca{I}(\rho^{\theta})<\infty$, otherwise the statement trivially follows. Throughout the proof we set $\Phi_{\theta}=\Phi(\rho^{\theta})$. Since $\ca{I}(\rho^{\theta})<\infty$, we can choose $g\in L^2([0,T]\times\Td)$ such that
    \begin{align}\label{eq:proof-analytical-heat-1}
        \partial_t\rho^{\theta}-\nu\Delta\rho^{\theta}=-\nabla\cdot\big(\Phi_{\theta}^{1/2}g\big),
    \end{align}
    is satisfied in the sense of \eqref{eq:skeleton-weak-sense}. As a consequence of the maximum principle for the heat equation, and the assumption on $\rho_0$ taking values away from zero, we have that
    \begin{align}\label{eq:log-rho-as-2}
        \log\rho\in L^2([0,T],H^1(\Td)).    
    \end{align}
    Decompose now,
    \begin{align*}
        \frac{d}{dt}\ca{H}(\rho^{\theta}(t)|\rho(t)) = \frac{d}{dt}\int_{\T^d}\rho^{\theta}(t)\log\rho^{\theta}(t)dx-\frac{d}{dt}\int_{\T^d}\rho^{\theta}(t)\log\rho(t)dx=\text{I} - \text{II},
    \end{align*}
    where the above integrals are finite as a consequence of equations \eqref{eq:log-rho-theta-as} and \eqref{eq:log-rho-as-2}, and that both the model and the ansatz are bounded functions. Since
    \begin{align}\label{eq:heat-analytical-dominated}
        |\partial_t\rho^{\theta}\log\rho^{\theta}| + |\partial_t\rho^{\theta}|\le\frac{1}{2}|\partial_t\rho^{\theta}|^2 + \frac{1}{2}|\log\rho^{\theta}|^2+\left(1+|\partial_t\rho^{\theta}|^2\right)\in L^1([0,T]\times\Td),
    \end{align}
       
    we can differentiate the parameter-dependent integrals, and use the fact that $\rho^{\theta}$ satisfies \eqref{eq:proof-analytical-heat-1} in a weak sense to obtain
    \begin{align*}
        \text{I}=&\int_{\Td}\partial_t\rho^{\theta}(t)\log\rho^{\theta}(t)dx + \int_{\Td}\partial_t\rho^{\theta}(t)dx\\
        =&\nu\int_{\Td}\Delta\rho^{\theta}(t)\log\rho^{\theta}(t)dx - \int_{\Td}\nabla\cdot\big(\Phi_{\theta}^{1/2}g\big)\log\rho^{\theta}(t)dx + \int_{\Td}\partial_t\rho^{\theta}(t)dx\\
        =&-\nu\int_{\Td}\frac{|\nabla\rho^{\theta}(t)|^2}{\rho^{\theta}(t)}dx + \int_{\Td}\Phi_{\theta}^{1/2}g\nabla\log\rho^{\theta}(t)dx.
    \end{align*}
    Here we have once again used that $\rho^{\theta}$ satisfies \eqref{eq:proof-analytical-heat-1}, which ensures that the integral of $\partial_t\rho^{\theta}$ vanishes. By analogous arguments we have that  
    \begin{align*}
        \text{II}=&-\nu\int_{\Td}\nabla\rho^{\theta}(t)\frac{\nabla\rho(t)}{\rho(t)}dx + \int_{\Td}\Phi^{1/2}_{\theta}g\nabla\log\rho(t) dx-\nu\int_{\Td}\frac{\nabla\rho^{\theta}(t)}{\rho(t)}\cdot\nabla\rho(t)dx\\
        &+\nu\int_{\Td}\frac{|\nabla\rho(t)|^2}{\rho(t)^2}\rho^{\theta}(t)dx.
    \end{align*}
    Collecting terms we get that for all $t\in[0,T]$
    \begin{align}\label{eq:proof-analytical-heat-2}
        \frac{d}{dt}\ca{H}(\rho^{\theta}(t)|\rho(t))+\nu\int_{\Td}\rho^{\theta}(t)\Bigg(\frac{\nabla\rho^{\theta}(t)}{\rho^{\theta}(t)}-\frac{\nabla\rho(t)}{\rho(t)}\Bigg)^2dx = \int_{\Td}\Phi_{\theta}^{1/2}g(t)\Bigg(\frac{\nabla\rho^{\theta}(t)}{\rho^{\theta}(t)}-\frac{\nabla\rho(t)}{\rho(t)}\Bigg)dx.
    \end{align}
    Similarly, one obtains
    \begin{align}\label{eq:proof-analytical-heat-3}
        \frac{d}{dt}\ca{H}(1-\rho(t)^{\theta}|1-\rho(t))&+\nu\int_{\Td}(1-\rho^{\theta}(t))\Bigg(\frac{\nabla\rho^{\theta}(t)}{1-\rho^{\theta}(t)}-\frac{\nabla\rho(t)}{1-\rho(t)}\Bigg)^2dx\nonumber\\
        &= -\int_{\Td}\Phi_{\theta}^{1/2}g(t)\Bigg(\frac{\nabla\rho^{\theta}(t)}{1-\rho^{\theta}(t)}-\frac{\nabla\rho(t)}{1-\rho(t)}\Bigg)dx.
    \end{align}
    We focus for a moment on \eqref{eq:proof-analytical-heat-2}. By H\"older's and Young's inequality with $\varepsilon>0$ to be chosen we get
    \begin{align*}
        \left|\int_{\Td}\Phi_{\theta}^{1/2}g(t)\bigg(\frac{\nabla\rho^{\theta}(t)}{\rho^{\theta}(t)}-\frac{\nabla\rho(t)}{\rho(t)}\bigg)dx\right|\le\frac{\varepsilon^{-1}}{2}\int_{\Td}|g(t)|^2dx + \frac{\varepsilon}{2}\int_{\Td}\Phi_{\theta}\bigg(\frac{\nabla\rho^{\theta}(t)}{\rho^{\theta}(t)}-\frac{\nabla\rho(t)}{\rho(t)}\bigg)^2dx.
    \end{align*}
    Given that $\rho^{\theta}\in(0,1)$, we have $\rho^{\theta}(1-\rho^{\theta})\le\rho^{\theta}$ and $\rho^{\theta}(1-\rho^{\theta})\le(1-\rho^{\theta})$. By taking $\varepsilon=\nu$ we obtain
    \begin{align}\label{eq:proof-analytical-heat-4}
        \frac{d}{dt}\ca{H}(\rho^{\theta}(t)|\rho(t))+\frac{\nu}{2}\int_{\Td}\rho^{\theta}(t)\Bigg(\frac{\nabla\rho^{\theta}(t)}{\rho^{\theta}(t)}-\frac{\nabla\rho(t)}{\rho(t)}\Bigg)^2dx \le \frac{\nu^{-1}}{2}\int_{\Td}\Phi_{\theta}|g(t)|^2dx.
    \end{align}
    Similarly, starting from \eqref{eq:proof-analytical-heat-3} we obtain
    \begin{align}\label{eq:proof-analytical-heat-5}
        \frac{d}{dt}\ca{H}(1-\rho^{\theta}(t)|1-\rho(t))&+\frac{\nu}{2}\int_{\Td}(1-\rho^{\theta}(t))\Bigg(\frac{\nabla\rho^{\theta}(t)}{(1-\rho^{\theta}(t))}-\frac{\nabla\rho(t)}{(1-\rho(t))}\Bigg)^2dx\nonumber\\
        &\le \frac{\nu^{-1}}{2}\int_{\Td}\Phi_{\theta}|g(t)|^2dx.
    \end{align}
    Finally, combining \eqref{eq:proof-analytical-heat-4} and \eqref{eq:proof-analytical-heat-5} we arrive at
    \begin{align*}
        \ca{H}(\rho^{\theta}(T)|\rho(T))+&\ca{H}(1-\rho^{\theta}(T)|1-\rho(T))+\frac{\nu}{2}\int_0^T\big(\ca{F}(\rho^{\theta}(t)|\rho(t))+\ca{F}(1-\rho^{\theta}(t)|1-\rho(t))\big)dt\\
        &\le\ca{H}(\rho^{\theta}(0)|\rho(0))+\ca{H}(1-\rho^{\theta}(0)|1-\rho(0))+\nu^{-1}\int_0^T\int_{\Td}|g(t)|^2dxdt.
    \end{align*}
    Since $g$ was taken as an arbitrary solution to \eqref{eq:proof-analytical-heat-1}, we conclude the result by taking the infimum over all such functions $g$'s.
\end{proof}

Now we turn to the numerical stability analysis, here the errors committed by the models not being periodic will become apparent. Recall that for a function $\varphi\colon Q_d\to\R^n$ for $n\in\{1,d\}$ we set
\begin{align*}
    \ell_T(\varphi)=\int_0^T\sum_{i=1}^d\int_{C_i}|\varphi(x+e_i)-\varphi(x)|^2dS(x).
\end{align*}
\begin{theorem}\label{thm:heat-numerical-estimate}
Let $\rho_0\in H^1(\Td)$ take values in $[\delta,1-\delta]$ for some $\delta>0$, and let $\rho\in L^2([0,T];H^2(\bb{T}^d))$ with $\partial_t\rho\in L^2([0,T],L^2(\Td))$ be the weak solution of \eqref{eq:heat}. Recall the definition of $\ca{M}$ given in \eqref{eq:M_def}. Let $\rho^{\theta}\in\ca{M}$ such that
\begin{align}\label{eq:log-rho-theta--num-as}
    \log\rho^{\theta}\in L^2([0,T],H^1(Q_d))\text{ and }\ca{F}(\rho^{\theta}(\cdot))\in L^2([0,T]).   
\end{align}
Then there exists a constant $\kappa=\kappa\left(d,\delta,\nu,T,\rho^{\theta},\rho\right)$, such that for all $t\in[0,T]$
    \begin{align*}
    \ca{H}^s(\rho^{\theta}(t)|\rho(t))+\frac{\nu}{2}\int_0^t
\ca{F}^s(\rho^{\theta}(s)|\rho(s))ds\le I_0(\rho_0^{\theta}) + \nu^{-1}I_{\text{dyn}}(\rho^{\theta})
    +\kappa\ca{D},
    \end{align*}
    where $\ca{D}$ accounts for the non-periodicity of the network and is given by
    \begin{align*}
        \ca{D}:=& C(1+\nu)(1+\delta^{-1})\big(\ell^{1/2}_T(\nabla\rho^{\theta})+\ell^{1/2}_T(\rho^{\theta})+\ell^{1/2}_T(r_{\theta})+\ell_T^{1/2}(\log\rho^{\theta})+\ell_T^{1/2}(\log(1-\rho^{\theta}))\big)\\
        &\times\big[\|\rho^{\theta}\|_{L^2([0,T],H^2(Q_d))}+\|\rho\|_{L^2([0,T],H^2(Q_d))}+\|r^{\theta}\|_{L^2([0,T],L^2(Q_d))}\\
        &+\|\log\rho^{\theta}\|_{L^2([0,T],H^1(Q_d))}+\|\log(1-\rho^{\theta})\|_{L^2([0,T],H^1(Q_d))}\big],
    \end{align*}
    and $r^{\theta}$ given as in \eqref{eq:r-theta}, which in this special case simplifies to
    \begin{align*}
        r^{\theta}=\partial_t w^{\theta}-\nu\Delta w^{\theta}.
    \end{align*}
    Here the Laplacian of a vector field is understood component-wise.
\end{theorem}
\begin{proof}
    The proof of this result proceeds similarly to Theorem \ref{thm:heat-analytical-estimate}. We therefore concentrate on the differences caused by the boundary terms.
    
    Fix $\rho^{\theta}\in\ca{M}$ with $\log\rho^{\theta}\in L^2([0,T],H^1(Q_d))$. Let now $\varepsilon>0$, and $g=g^{\varepsilon}\in L^2([0,T]\times Q_d)$ an admissible control as in \eqref{eq:dyn-rate-numerical} such that
    \begin{align}\label{eq:proof-numerical-heat-1}
        \partial_t\rho^{\theta}-\nu\Delta\rho^{\theta}&=-\nabla\cdot(\Phi_{\theta}^{1/2}g)\\
        \Phi^{1/2}_{\theta}g\cdot\vec{n}(x)&=r^{\theta}\cdot\vec{n}(x),\nonumber
    \end{align}
    is satisfied in a weak sense and
    \begin{align*}
        \int_0^T\int_{Q_d}|g|^2dxdt\le I_{\text{dyn}}(\theta)+\varepsilon.
    \end{align*}
    The existence of such control is guaranteed by \eqref{eq:I-finite}. We decompose,
    \begin{align*}
        \frac{d}{dt}\ca{H}(\rho^{\theta}(t)|\rho(t)) = \frac{d}{dt}\int_{Q_d}\rho^{\theta}(t)\log\rho^{\theta}(t)dx-\frac{d}{dt}\int_{Q_d}\rho^{\theta}(t)\log\rho(t)dx=\text{I} - \text{II}.
    \end{align*}
    By differentiation of parameterized integrals, see \eqref{eq:heat-analytical-dominated} in the proof of Theorem \ref{thm:heat-analytical-estimate}, and using that $\rho^{\theta}$ satisfies \eqref{eq:proof-numerical-heat-1} 
    \begin{align*}
        \text{I}=&\nu\int_{Q_d}\Delta\rho^{\theta}(t)\log\rho^{\theta}(t)dx - \int_{Q_d}\nabla\cdot\big(\Phi_{\theta}^{1/2}g\big)\log\rho^{\theta}(t)dx + \int_{Q_d}\partial_t\rho^{\theta}(t)dx\\
        =&-\nu\int_{Q_d}\frac{|\nabla\rho^{\theta}(t)|^2}{\rho^{\theta}(t)}dx + B_1(t) + \int_{Q_d}\Phi_{\theta}^{1/2}g\cdot \frac{\nabla\rho^{\theta}(t)}{\rho^{\theta}(t)}dx - B_2(t) + B_3(t) - B_4(t),
    \end{align*}
    where
    \begin{align*}
        B_1(t) &= \nu\int_{\partial Q_d}\nabla\rho^{\theta}(t)\log\rho^{\theta}(t) \cdot \vec{n}(x)dS(x),\\
        B_2(t) &= \int_{\partial Q_d}\Phi^{1/2}_{\theta}g\cdot\vec{n}(x)\log\rho^{\theta}dS(x),\\
        B_3(t) &= \nu\int_{\partial Q_d}\nabla\rho^{\theta}(t)\cdot\vec{n}(x)dS(x),\\
        B_4(t) &=\int_{\partial Q_d}\Phi_{\theta}^{1/2}g\cdot \vec{n}(x)dS(x).
    \end{align*}
    Since $\log\rho$ and $\log\rho^{\theta}$ are $L^2([0,T],H^1(Q_d))$ functions, their values on the boundary are well defined by Sobolev embedding. On the other hand, by the same arguments one obtains
    \begin{align*}
        \text{II}=&-\nu\int_{Q_d}\nabla\rho^{\theta}(t)\cdot\frac{\nabla\rho(t)}{\rho(t)}dx + B_5(t) + \int_{Q_d} \Phi^{1/2}_{\theta}g_t\cdot\frac{\nabla\rho(t)}{\rho(t)} dx-B_6(t)-\nu\int_{Q_d}\frac{\nabla\rho^{\theta}(t)}{\rho(t)}\cdot\nabla\rho(t)dx\\
        &+\nu\int_{Q_d}\frac{|\nabla\rho(t)|^2}{\rho(t)^2}\rho^{\theta}(t)dx+B_7(t),
    \end{align*}
    where
    \begin{align*}
        B_5(t)=&\nu\int_{\partial Q_d}\nabla\rho^{\theta}(t)\log\rho(t)\cdot\vec{n}dS(x),\\
        B_6(t)=&\int_{\partial Q_d}\Phi_{\theta}^{1/2}g\log\rho(t)\cdot \vec{n}dS(x),\\
        B_7(t)=&\nu\int_{\partial Q_d}\frac{\rho^{\theta}(t)}{\rho(t)}\nabla\rho(t)\cdot\vec{n}dS(x).
    \end{align*}
    Combining I and II, we get
    \begin{align}\label{eq:proof-numerical-heat-2}
        \frac{d}{dt}\ca{H}(\rho^{\theta}(t)|\rho(t))=&-\nu\int_{Q_d}\rho^{\theta}(t)\Bigg(\frac{\nabla\rho^{\theta}(t)}{\rho(t)^{\theta}}-\frac{\nabla\rho(t)}{\rho(t)}\Bigg)^2dx+\int_{Q_d}\Phi^{1/2}_{\theta(t)}g\left(\frac{\nabla\rho^{\theta}(t)}{\rho^{\theta}(t)}-\frac{\nabla\rho(t)}{\rho(t)}\right)dx\nonumber\\
        &+\sum_{i=1}^7(-1)^{i+1}B_i(t).
    \end{align}
    Note that for $i=1,...,7$, we have $B_i(t)=B_i(\rho(t),\rho^{\theta}(t))$. Similarly, one can compute the time derivative of $\ca{H}(1-\rho^{\theta}(t)|1-\rho(t))$, and obtain
    \begin{align}\label{eq:proof-numerical-heat-3}
        \frac{d}{dt}\ca{H}(1-\rho^{\theta}(t)|1-\rho(t))=&-\nu\int_{Q_d}(1-\rho^{\theta}(t))\Bigg(\frac{\nabla\rho^{\theta}(t)}{1-\rho(t)^{\theta}}-\frac{\nabla\rho(t)}{1-\rho(t)}\Bigg)^2dx\\
        &+\int_{Q_d}\Phi^{1/2}_{\theta}g\left(\frac{\nabla\rho^{\theta}(t)}{1-\rho^{\theta}(t)}-\frac{\nabla\rho(t)}{1-\rho(t)}\right)dx+\sum_{i=1}^7(-1)^{i+1}\widetilde{B}_i(t),\nonumber
    \end{align}
    where $\widetilde{B}_i(t)=B_i(1-\rho(t),1-\rho^{\theta}(t))$. Note that $B_i(t)+\widetilde{B}_i(t)=0$ for $i=3,4$. By adding \eqref{eq:proof-numerical-heat-2} and \eqref{eq:proof-numerical-heat-3}, integrating with respect to $s\in(0,t)$, and using Young's inequality with $\varepsilon=\nu/2$, we obtain that for all $0\le t\le T$,
    \begin{align}\label{eq:proof-numerical-heat-4}
        &\ca{H}(\rho^{\theta}(t)|\rho(t))+\frac{\nu}{2}\int_0^t\ca{F}(\rho^{\theta}(s)|\rho(s))ds+\ca{H}(1-\rho^{\theta}(t)|1-\rho(t))+\frac{\nu}{2}\int_0^t\ca{F}(1-\rho^{\theta}(s)|1-\rho(s))ds\nonumber\\
        &\le 4\nu^{-1}\int_0^T\int_{Q_d}|g|^2dxds+I_0(\theta)+\sum_{i=1,
        i\neq 3,4}^7(-1)^{i+1}\int_0^T|B_i(s)+\widetilde{B}_i(s)|ds.
    \end{align}
    We next estimate the boundary term $B_2(t)$ for $t \in [0,T]$. The estimates for the remaining boundary contributions follow along the same lines. Recall that $\{C_i\}_{i=1}^d$ denote the facets of $Q_d$ such that $0 \in C_i$. We then have
    \begin{align*}
        \left|\int_0^T B_2(t)\, dt\right|
        &= \left|\int_0^T \int_{\partial Q_d} 
            \Phi^{1/2}_{\theta} g \cdot \vec{n}(x) \log\rho^{\theta}(t) 
            \, dS(x) \, dt\right| \\
        &= \left|\int_0^T \int_{\partial Q_d} 
            r_{\theta}\cdot \vec{n}(x)\, \log\rho^{\theta}(t) 
            \, dS(x) \, dt\right| \\
        &\leq \sum_{i=1}^d \int_0^T \int_{C_i} 
            \big| r^{(i)}_{\theta}(t,x+e_i) - r^{(i)}_{\theta}(t,x)\big|\,
            \big|\log\rho^{\theta}(t,x+e_i)\big| 
            \, dS(x) \, dt \\
        &\quad + \sum_{i=1}^d \int_0^T \int_{C_i} 
            \big|\log\rho^{\theta}(t,x+e_i)-\log\rho^{\theta}(t,x)\big|\,
            \big|r^{(i)}_{\theta}(t,x)\big| 
            \, dS(x) \, dt \\
        &\leq \ell_T(r_{\theta})^{1/2}\,
            \|\log\rho^{\theta}\|_{L^2([0,T];H^1(Q_d))}
            + \ell_T(\log\rho^{\theta})^{1/2}\,
            \|r_{\theta}\|_{L^2([0,T];L^2(\partial Q_d))},
    \end{align*}
    where we used the boundary condition satisfied by $\Phi^{1/2}_{\theta} g$, Hölder’s inequality, and the trace theorem to control the integrals over the boundary.  

    Since, by assumption $\rho_0$ is bounded away from zero, by  the maximum principle for the heat equation, the same is true for the solution $\rho$. This will be used in the treatment of the terms $B_7$ and $\widetilde{B_7}$. For the remaining boundary terms we obtain the following estimates:
    \begin{align*}
        \left|\int_0^T(|B_1(t)|+|\widetilde{B}_1(t)|)dt\right|\le&\nu\ell_T(\rho^{\theta})^{1/2}\left(\|\log\rho^{\theta}\|_{L^2([0,T],H^1(Q_d))}+\|\log(1-\rho^{\theta})\|_{L^2([0,T],H^1(Q_d))}\right)\\
        &+\nu\|\rho^{\theta}\|_{L^2([0,T],H^2(Q_d))}\left(\ell_T(\log\rho^{\theta})^{1/2}+\ell_T(\log(1-\rho^{\theta}))^{1/2}\right),\\
        \left|\int_0^T\widetilde{B}_2(t)dt\right|\le&\ell_T(r_{\theta})^{1/2}\|\log(1-\rho^{\theta})\|_{L^2([0,T],H^1(Q_d))}\\
        &+\ell_T(\log(1-\rho^{\theta}))^{1/2}\|r_{\theta}\|_{L^2([0,T],L^2(\partial Q_d))},\\
        \int_0^T(|B_5(t)|+|\widetilde{B}_5(t)|)dt\le&\nu\ell_T(\nabla\rho^{\theta})^{1/2}\big[\|\log\rho\|_{L^2([0,T],H^1(Q_d))}+\|\log(1-\rho)\|_{L^2([0,T],H^1(Q_d))}\big],\\
        \int_0^T(|B_6(t)|+|\widetilde{B}_6(t)|)dt\le&\ell_T(r_{\theta})^{1/2}\big[\|\log\rho\|_{L^2([0,T],H^1(Q_d))}+\|\log(1-\rho)\|_{L^2([0,T],H^1(Q_d))}\big],\\
        \int_0^T(|B_7(t)|+|\widetilde{B}_7(t)|)dt\le&2\delta^{-1}\nu\ell_T(\rho^{\theta})^{1/2}\|\rho\|_{L^2([0,T],H^2(Q_d))}.
    \end{align*}
    By collecting similar contributions, the rightmost term in \eqref{eq:proof-numerical-heat-4} can be bounded by  
    \begin{align*}
        \ca{D}:=& C(1+\nu)(1+\delta^{-1})\big(\ell^{1/2}_T(\nabla\rho^{\theta})+\ell^{1/2}_T(\rho^{\theta})+\ell^{1/2}_T(r_{\theta})+\ell_T^{1/2}(\log\rho^{\theta})+\ell_T^{1/2}(\log(1-\rho^{\theta}))\big)\\
        &\times\big[\|\rho^{\theta}\|_{L^2([0,T],H^2(Q_d))}+\|\rho\|_{L^2([0,T],H^2(Q_d))}+\|r^{\theta}\|_{L^2([0,T],L^2(Q_d))}\\
        &+\|\log\rho^{\theta}\|_{L^2([0,T],H^1(Q_d))}+\|\log(1-\rho^{\theta})\|_{L^2([0,T],H^1(Q_d))}\big].
    \end{align*}
    It follows that the left-hand side of \eqref{eq:proof-numerical-heat-4} is upper bounded by
    \begin{align*}
        4\nu^{-1}(I_{\text{dyn}}(\theta)+\varepsilon)+I_0(\theta)+\ca{D},
    \end{align*}
    for all $\varepsilon>0$ and uniformly with respect to $t\in[0,T]$. This completes the proof.
\end{proof}
The following corollary follows directly from Theorem \ref{thm:heat-numerical-estimate}.
\begin{corollary}
    Under the assumption of Theorem \ref{thm:heat-numerical-estimate}, there exists a constant 
    \begin{align*}
        \kappa=\kappa\left(d,\delta,\nu,T,w^{\theta},\rho\right),
    \end{align*}
    such that
    \begin{align*}
    \int_0^T\left(\int_{Q_d}|\rho^{\theta}(t)-\rho(t)|dx\right)^2dt\le I_0(\rho_0^{\theta}) + \nu^{-1}I_{\text{dyn}}(\rho^{\theta})
    +\kappa\ca{D},
    \end{align*}
    where $\ca{D}$ was made explicit in Theorem \ref{thm:heat-numerical-estimate}.
\end{corollary}
\begin{proof}
    Follows from Theorem \ref{thm:l1-entropy}.
\end{proof}

\subsection{Viscous Burgers' equation}
In this section we prove stability estimates for a one dimensional hyperbolic conservation law of the form
\begin{align}\label{eq:burgers-posteriori}
    \partial_t\rho&=\nu\partial_{xx}^2\rho-\partial_x(f(\rho)),\quad \rho(0)=\rho_0,
\end{align}
posed on $[0,T]\times\bb{T}$ and $f\colon\R\to\R$ a sufficiently smooth function. 

The underlying physical system that gives rise to \eqref{eq:burgers-posteriori} is the weakly asymmetric exclusion process (WASEP), see e.g. \cite{kipnis89}, and roughly speaking, corresponds to a perturbation of the SSEP in which particles are biased to move towards a particular direction. This perturbation then generates interesting non-linear effects, for example the approximate formation shocks of the hydrodynamic limit \eqref{eq:burgers-posteriori}. The weight function in this case the same as in the SSEP, $\Phi(x)=x(1-x)$, but with the residual will be different.

We make use of the following existence result whose proof follows from standard arguments, we add it for the sake of completeness. See also \cite{beania-sadallah16}.\\

\begin{prop}\label{prop}
    Let $\rho_0\in H^2(\bb{T})\cap\{0\le\rho\le 1\}$, and assume that $f\in C^1(\bb{R})$ is a Lipchitz function with constant $[f]_L$. Then there exists a unique weak solution $\rho\in C([0,T],L^2(\bb{T}))\cap L^2([0,T],H^2(\bb{T}))$ of \eqref{eq:burgers-posteriori} with $0\le\rho\le 1$. 
\end{prop}

As in the previous section, we first present a result concerning periodic models.\\

\begin{theorem}\label{thm:burgers-analytical-estimate}
    Let $\rho_0$ take values in $[\delta,1-\delta]$, and let $\rho\in C([0,T],L^2(\bb{T}))\cap L^2([0,T],H^2(\bb{T}))$ be the weak solution of \eqref{eq:burgers-posteriori} with $f(x)=x(1-x)$ such that for some $\delta>0$, $\delta\le\rho\le1-\delta$ almost everywhere. Consider a parametrized set of periodic functions
\begin{align}\label{eq:bur-analytical-setN}
    \ca{N}=\left\{\rho^{\theta}\colon[0,T]\times\Td\to(0,1)\ \big|\ \rho^{\theta}\in D([0,T],M),\text{ and }\theta\in\Theta\right\},
\end{align}
where $\Theta\subset\R^p$, $p\in\N$. Let $\rho^{\theta}\in\ca{N}$ such that
\begin{align}\label{eq:bur-log-rho-theta-as}
    \log\rho^{\theta}\in L^2([0,T],H^1(\Td))\text{ and }\ca{F}(\rho^{\theta}(\cdot))\in L^2([0,T]).   
\end{align}
Further, assume
\begin{align}\label{eq:bur-log-rho-as}
    \ca{F}(\rho(\cdot))\in L^2([0,T]).
\end{align}
Then 
    \begin{align*}
        \ca{H}^s(\rho^{\theta}(t)|\rho(t))+\frac{\nu}{2}\int_0^t\ca{F}^s(\rho^{\theta}(s)|\rho(s))ds\le \ca{I}_0(\rho^{\theta})+ C\ca{I}_0(\rho^{\theta})^{1/2}
        +\ca{I}_{\text{dyn}}(\rho^{\theta})\nu^{-1}\left(1+\frac{\nu^{-1}C}{2}\right),
    \end{align*}
    where $C=\exp(\nu^{-1}T[f]_L/2)$ and $[f]_L$ is the Lipschitz constant of $f$.
\end{theorem}
\begin{proof}
    The proof proceeds similarly to Theorem \ref{thm:heat-analytical-estimate}. Take $\rho^{\theta}\in\ca{N}$ such that $\ca{I}(\rho^{\theta})<\infty$, otherwise the statement trivially follows. As a consequence, for a fixed $\varepsilon>0$ there exists $g_{\varepsilon}\in L^{2}([0,T]\times\T)$ satisfying the skeleton equation and such that
    \begin{align}\label{eq:burgers-analytical-proof-0}
        \int_0^T\int_{\T}|g|^2dxdt\le\ca{I}_{\text{dyn}}(\rho^{\theta}) + \varepsilon.
    \end{align}
    By the same arguments as in the first step of the proof of Theorem \ref{thm:heat-analytical-estimate}, we have that
    \begin{align}\label{eq:burgers-analytical-proof-1}
        \frac{d}{dt}\ca{H}(\rho^{\theta}(t)|\rho(t))&+\nu\int_{\T}\left(\frac{\partial_x\rho^{\theta}(t)}{\rho^{\theta}(t)}-\frac{\partial_x\rho(t)}{\rho(t)}\right)^2\rho^{\theta}(t)dx=\int_{\T}\Phi^{1/2}g\left(\frac{\partial_x\rho^{\theta}(t)}{\rho^{\theta}(t)}-\frac{\partial_x\rho(t)}{\rho(t)}\right)dx\\
        &+\int_{\T}f(\rho^{\theta}(t))\frac{\partial_x\rho^{\theta}(t)}{\rho^{\theta}(t)}dx-\int_{\T}f(\rho^{\theta}(t))\frac{\partial_x\rho(t)}{\rho(t)}dx\\
        \nonumber&+\int_{\T}\rho^{\theta}(t)\frac{\partial_x\rho(t)}{\rho(t)^2}f(\rho(t))dx-\int_{\T}\partial_x\rho^{\theta}(t)\frac{f(\rho(t))}{\rho(t)}dx.\nonumber
    \end{align}
    In addition we have,
    \begin{align}\label{eq:burgers-analytical-proof-2}
        \frac{d}{dt}&\ca{H}(1-\rho^{\theta}(t)|1-\rho(t))+\nu\int_{\T}\left(\frac{\partial_x\rho^{\theta}(t)}{1-\rho^{\theta}(t)}-\frac{\partial_x\rho(t)}{1-\rho(t)}\right)^2(1-\rho^{\theta}(t))dx\\
        &=\int_{\T}\Phi^{1/2}g\left(\frac{\partial_x\rho^{\theta}(t)}{1-\rho^{\theta}(t)}-\frac{\partial_x\rho(t)}{1-\rho(t)}\right)dx\nonumber+\int_{\T}f(\rho^{\theta}(t))\frac{\partial_x\rho^{\theta}(t)}{1-\rho^{\theta}(t)}dx-\int_{\T}f(\rho^{\theta}(t))\frac{\partial_x\rho(t)}{1-\rho(t)}dx\nonumber\\
        &+\int_{\T}(1-\rho^{\theta}(t))\frac{\partial_x\rho(t)}{(1-\rho(t))^2}f(\rho(t))dx-\int_{\T}\partial_x\rho^{\theta}(t)\frac{f(\rho(t))}{1-\rho(t)}dx.\nonumber
    \end{align}
    The definition of the function $f$ implies that the last four terms in \eqref{eq:burgers-analytical-proof-1} can be rewritten as    
    \begin{align*}
        \int_{\T}\left(\frac{\partial\rho(t)}{\rho(t)}-\frac{\partial\rho^{\theta}(t)}{\rho^{\theta}(t)}\right)\rho^{\theta}(t)(\rho(t)-\rho^{\theta}(t))dx.
    \end{align*}
    Similarly, the last four term in \eqref{eq:burgers-analytical-proof-2} can be rewritten as
    \begin{align*}
        \int_{\T}\left(\frac{\partial\rho(t)}{1-\rho(t)}-\frac{\partial\rho^{\theta}(t)}{1-\rho^{\theta}(t)}\right)(1-\rho^{\theta}(t))(\rho^{\theta}(t)-\rho(t))dx.
    \end{align*}
    After integrating over $s\in[0,t]$ for $t\le T$, and using Young's inequality with $\eta\in(0,\nu)$ it can be verified that for all such $t\le T$,
    \begin{align}\label{eq:burgers-analytical-proof-3}
        &\ca{H}(\rho^{\theta}(t)|\rho(t))-\ca{H}(\rho^{\theta}(t)|\rho(t)) + (\nu-\eta)\int_0^t\int_{\T}\left(\frac{\partial_x\rho^{\theta}(s)}{\rho^{\theta}(s)}-\frac{\partial_x\rho(s)}{\rho(s)}\right)^2\rho^{\theta}(s)dxds\nonumber\\
        &\le\frac{\eta^{-1}}{2}\int_0^t\int_{\T}|g|^2dxds+\frac{\eta^{-1}}{2}\int_0^t\int_{\T}|\rho^{\theta}(s)-\rho(s)|^2dxds.
    \end{align}
    Analogously,
    \begin{align}\label{eq:burgers-analytical-proof-4}
        &\ca{H}(1-\rho^{\theta}(t)|1-\rho(t))-\ca{H}(1-\rho^{\theta}(0)|1-\rho(0)) \nonumber\\
        &+ (\nu-\eta)\int_0^t\int_{\T}\left(\frac{\partial_x\rho^{\theta}(s)}{1-\rho^{\theta}(s)}-\frac{\partial_x\rho(s)}{1-\rho(s)}\right)^2(1-\rho^{\theta}(s))dxds\nonumber\\
        &\le\frac{\eta^{-1}}{2}\int_0^t\int_{\T}|g|^2dxds+\frac{\eta^{-1}}{2}\int_0^t\int_{\T}|\rho^{\theta}(s)-\rho(s)|^2dxds.
    \end{align}
    The last term in the above equation can be controlled by standard arguments as in \cite{ryck-mishra24,ryck-jagtap-mishra23}, the difference in our case is that we use that $\rho^{\theta}$ satisfies the skeleton equation with control $g=g_{\varepsilon}$. One obtains then that for all $t\le T$
    \begin{align}\label{eq:burgers-analytical-proof-5}
        \int_0^t\int_{\T}|\rho_s-\rho_s^{\theta}|^2dxds \le \left(\int_{\T}|\rho_0-\rho_0^{\theta}|^2dx + \frac{\nu^{-1}}{2}\int_0^t\int_{\T}|g|^2dxds\right)\exp\left(\frac{\nu^{-1}T[f]_L}{2}\right),
    \end{align}
    where $[f]_L$ is the Lipschitz constant of $f$. By replacing \eqref{eq:burgers-analytical-proof-5} and \eqref{eq:burgers-analytical-proof-0} into \eqref{eq:burgers-analytical-proof-3} and \eqref{eq:burgers-analytical-proof-4}, taking $\eta=\nu/2$ and adding up the corresponding results, we get that for all $t\le T$
    \begin{align*}
        &\ca{H}(\rho^{\theta}(t)|\rho(t))+\ca{H}(1-\rho^{\theta}(t)|1-\rho(t))+\frac{\nu}{2}\int_0^t\left(\ca{F}(\rho^{\theta}(s)|\rho(s))+\ca{F}(1-\rho^{\theta}(s)|1-\rho(s))\right)ds\\
        &\le \ca{I}_0(\rho^{\theta})+C\int_{\T}|\rho(0)-\rho^{\theta}(0)|^2dx
        +(\ca{I}_{\text{dyn}}(\rho^{\theta})+\varepsilon)\nu^{-1}\left(1+\frac{\nu^{-1}C}{2}\right),
    \end{align*}
    where $C=\exp(\nu^{-1}T[f]_L/2)$. Applying Pinsker's inequality \eqref{eq:pinsker} concludes the proof.
\end{proof}

In the proof of the numerical stability estimate, we make use of the following notation. Given $\varphi\colon[0,1]\to\R$ measurable, we denote
\begin{align}\label{eq:notation-ell-per}
    \ell_{\text{per}}(\varphi)=\varphi(1)-\varphi(0).
\end{align}
\begin{theorem}\label{thm:burgers-numerical-estimate}
Let $\rho\in L^2([0,T];H^2(\bb{T}^d))\cap C([0,T],L^2(\Td))$ be the weak solution of \eqref{eq:burgers-posteriori} with $\delta\le\rho_0\le1-\delta$ for some $\delta>0$ and such that $\ca{F}(\rho(\cdot))\in L^2([0,T])$. Recall the definition of $\ca{M}$ given in \eqref{eq:M_def}. Let $\rho^{\theta}\in\ca{M}$ such that 
\begin{align}\label{eq:bur-log-rho-theta-num-as}
    \log\rho^{\theta}\in L^2([0,T],H^1(\Td))\text{ and }\ca{F}(\rho^{\theta}(\cdot))\in L^2([0,T]).   
\end{align}
Then there exists a constant
    \begin{align*}
        \kappa=\kappa\left(d,\delta,\nu,T,\rho^{\theta},\rho\right),
    \end{align*}
    that grows exponential with respect to $\nu^{-1}$ and such that for all $t\in[0,T]$
    \begin{align*}
    \ca{H}(\rho^{\theta}(t)|\rho(t))+\ca{H}(1-\rho^{\theta}(t)|1-\rho(t))&+\frac{\nu}{2}\int_0^t
\big(\ca{F}(\rho^{\theta}(s)|\rho(s))+\ca{F}(1-\rho^{\theta}(s)|1-\rho(s))\big)ds\\
    &\le I_0(\rho_0^{\theta})+CI_0(\rho_0^{\theta})^{1/2} + \nu^{-1}I_{\text{dyn}}(\rho^{\theta})
    +\kappa\ca{D},
    \end{align*}
    where $C$ grows exponentially w.r.t $\nu^{-1}$, and $\ca{D}$ accounts for the non-periodicity of the network and is given by
    \begin{align*}
        \ca{D}:=& \exp\left(T\nu^{-1}[[f]_L+1]\right)(1+\nu)(1+\delta^{-1})\\
        &\times\big(\ell^{1/2}_T(\partial_x\rho^{\theta})+\ell^{1/2}_T(\rho^{\theta})+\ell^{1/2}_T(r_{\theta})+\ell_T^{1/2}(\log\rho^{\theta})+\ell_T^{1/2}(\log(1-\rho^{\theta}))\big)\\
        &\times\bigg[\|\rho^{\theta}\|_{L^2([0,T],H^2([0,1]))}+\|\rho\|_{L^2([0,T],H^2([0,1]))}+\|r^{\theta}\|_{L^2([0,T],L^2([0,1]))}\\
        &+\|\log\rho^{\theta}\|_{L^2([0,T],H^1([0,1]))}+\|\log(1-\rho^{\theta})\|_{L^2([0,T],H^1([0,1]))}\bigg],
    \end{align*}
    where $r^{\theta}$ is defined in \eqref{eq:r-theta}, and in this case it corresponds to
    \begin{align*}
        r^{\theta}=\partial_t w^{\theta}-\nu\partial_{xx}^2w^{\theta}-\rho^{\theta}(1-\rho^{\theta}).
    \end{align*}
\end{theorem}

\begin{proof}
    As in the case of the heat equation, here we pay special attention to the boundary terms appearing after integrating by parts. The same arguments mentioned in Theorem \ref{thm:burgers-analytical-estimate} will apply here. Note that here $Q_d$ is equal to the unit interval $[0,1]$.

    Fix $\rho^{\theta}\in\ca{M}$ with $\log\rho^{\theta}\in L^2([0,T],H^1(Q_d))$. Let now $\varepsilon>0$, and $g=g^{\varepsilon}\in L^2([0,T]\times Q_d)$ an admissible control as in \eqref{eq:dyn-rate-numerical} satisfying the equation
    \begin{align}\label{eq:burgers-numerical-proof-1}
        \partial_t\rho^{\theta}-\nu\partial_{xx}^2\rho^{\theta}&=\partial_xf(\rho^{\theta})-\partial_x(\Phi_{\theta}^{1/2}g)\\
        \Phi^{1/2}_{\theta}g\cdot\vec{n}(x)&=r^{\theta}\cdot\vec{n}(x),\nonumber
    \end{align}
    and
    \begin{align}\label{eq:burgers-numerical-proof-1.5}
        \int_0^T\int_{Q_d}|g|^2dxdt\le I_{\text{dyn}}(\theta)+\varepsilon.
    \end{align}
    Once again, the existence of such function $g$ is guaranteed by \eqref{eq:I-finite}. The computation of the time derivative of the function $t\mapsto\ca{H}(\rho^{\theta}(t)|\rho(t))+\ca{H}(1-\rho^{\theta}(t)|1-\rho(t))$ can be argued in the same way as in the first step of Theorem \ref{thm:heat-analytical-estimate}. Consequently, we obtain
    \begin{align}\label{eq:burgers-numerical-proof-2}
        &\frac{d}{dt}\ca{H}(\rho^{\theta}(t)|\rho(t))+\nu\int_{Q_d}\rho^{\theta}(t)\left(\frac{\partial_x\rho^{\theta}(t)}{\rho^{\theta}(t)}-\frac{\partial_x\rho(t)}{\rho(t)}\right)^2dx=\int_{Q_d}\Phi^{1/2}_{\theta}g\left(\frac{\partial_x\rho^{\theta}(t)}{\rho^{\theta}(t)}-\frac{\partial_x\rho(t)}{\rho(t)}\right)dx\nonumber\\
        &+\int_{Q_d}\rho^{\theta}(t)\left(\frac{\partial_x\rho^{\theta}(t)}{\rho^{\theta}(t)}-\frac{\partial_x\rho(t)}{\rho(t)}\right)\left(\frac{f(\rho^{\theta}(t))}{\rho^{\theta}(t)}-\frac{f(\rho(t))}{\rho(t)}\right)dx+ D_1(t),
    \end{align}
    and 
    \begin{align}\label{eq:burgers-numerical-proof-3}
        &\frac{d}{dt}\ca{H}(1-\rho^{\theta}(t)|1-\rho(t))+\nu\int_{Q_d}(1-\rho^{\theta}(t))\left(\frac{\partial_x\rho^{\theta}(t)}{1-\rho^{\theta}(t)}-\frac{\partial_x\rho(t)}{1-\rho(t)}\right)^2dx\nonumber\\
        &=\int_{Q_d}\Phi^{1/2}_{\theta}g\left(\frac{\partial_x\rho^{\theta}(t)}{1-\rho^{\theta}(t)}-\frac{\partial_x\rho(t)}{1-\rho(t)}\right)dx\nonumber\\
        &+\int_{Q_d}(1-\rho^{\theta}(t))\left(\frac{\partial_x\rho^{\theta}(t)}{1-\rho^{\theta}(t)}-\frac{\partial_x\rho(t)}{1-\rho(t)}\right)\left(\frac{f(\rho^{\theta}(t))}{1-\rho^{\theta}(t)}-\frac{f(\rho(t))}{1-\rho(t)}\right)dx+ D_2(t).
    \end{align}
    The time dependent functions $D_1$ and $D_2$ correspond to the boundary terms coming from the integration by parts. These are given by
    \begin{align}\label{eq:burgers-numerical-proof-D1}
        D_1(t)=&\ell_{\text{per}}(\log(\rho^{\theta}(t)))\left[-\nu\partial_x\rho^{\theta}(t,1)+r^{\theta}(t,1)+f(\rho^{\theta}(t,1))\right]\nonumber\\
        &+\ell_{\text{per}}(\partial_x\rho^{\theta}(t))\left[\nu\log(\rho^{\theta}(t,0))-\nu+\nu\log(\rho(t,0))\right]\nonumber\\
        &+\left(\ell_{\text{per}}(f(\rho^{\theta}(t))+\ell_{\text{per}}(r^{\theta}(t))\right)\left[\log(\rho^{\theta}(t,0))+1-\log(\rho(t,0))\right]\nonumber\\
        &+\ell_{\text{per}}(\rho^{\theta}(t))\left[\frac{f(\rho(t,0))}{\rho(t,0)}-\nu\frac{\partial_x\rho(t,0)}{\rho(t,0)}\right],
    \end{align}
    and
    \begin{align}\label{eq:burgers-numerical-proof-D2}
        D_2(t)=&\ell_{\text{per}}(\log(1-\rho^{\theta}(t)))\left[-\nu\partial_x\rho^{\theta}(t,1)+r^{\theta}(t,1)+f(\rho^{\theta}(t,1))\right]\nonumber\\
        &+\ell_{\text{per}}(\partial_x\rho^{\theta}(t))\left[\nu\log(1-\rho^{\theta}(t,0))-\nu+\nu\log(1-\rho(t,0))\right]\nonumber\\
        &+\left(\ell_{\text{per}}(f(\rho^{\theta}(t))+\ell_{\text{per}}(r^{\theta}(t))\right)\left[\log(1-\rho^{\theta}(t,0))+1-\log(1-\rho(t,0))\right]\nonumber\\
        &+\ell_{\text{per}}(\rho^{\theta}(t))\left[\frac{f(\rho(t,0))}{1-\rho(t,0)}-\nu\frac{\partial_x\rho(t,0)}{1-\rho(t,0)}\right].
    \end{align}
    We focus for a moment on \eqref{eq:burgers-numerical-proof-2}. After integrating \eqref{eq:burgers-numerical-proof-2} over $s\in[0,t]$ for $t\le T$ we arrive at
    \begin{align*}
        &\ca{H}(\rho^{\theta}(t)|\rho(t))-\ca{H}(\rho^{\theta}(0)|\rho(0))+\nu\int_0^t\int_{Q_d}\rho^{\theta}(s)\left(\frac{\partial_x\rho^{\theta}(s)}{\rho^{\theta}(s)}-\frac{\partial_x\rho(s)}{\rho(s)}\right)^2dxds\nonumber\\
        &=\int_0^t\int_{Q_d}\Phi^{1/2}_{\theta}g\left(\frac{\partial_x\rho^{\theta}(s)}{\rho^{\theta}(s)}-\frac{\partial_x\rho(s)}{\rho(s)}\right)dsdx\\
        &+\int_0^t\int_{Q_d}\rho^{\theta}(s)\left(\frac{\partial_x\rho^{\theta}(s)}{\rho^{\theta}(s)}-\frac{\partial_x\rho(s)}{\rho(s)}\right)\left(\frac{f(\rho^{\theta}(s))}{\rho^{\theta}(s)}-\frac{f(\rho(s))}{\rho(s)}\right)dsdx\\
        &+\int_0^tD_1(s)ds.
    \end{align*}
    We now use Young's inequality with $\eta\in(0,\nu)$ to be chosen and obtain 
    \begin{align*}
        &\ca{H}(\rho^{\theta}(t)|\rho(t))-\ca{H}(\rho^{\theta}(0)|\rho(0))+\nu\int_0^t\int_{Q_d}\rho^{\theta}(s)\left(\frac{\partial_x\rho^{\theta}(s)}{\rho^{\theta}(s)}-\frac{\partial_x\rho(s)}{\rho(s)}\right)^2dxds\\
        &\le\frac{\eta^{-1}}{2}\int_0^t\int_{Q_d}|g|^2dsdx+\frac{\eta}{2}\int_0^t\int_{Q_d}\Phi_{\theta}\left(\frac{\partial_x\rho^{\theta}(s)}{\rho^{\theta}(s)}-\frac{\partial_x\rho(s)}{\rho(s)}\right)^2dsdx\nonumber\\
        &+\frac{\eta}{2}\int_0^t\int_{Q_d}\rho^{\theta}\left(\frac{\partial_x\rho^{\theta}(s)}{\rho^{\theta}(s)}-\frac{\partial_x\rho(s)}{\rho(s)}\right)^2dsdx\\
        &+\frac{\eta^{-1}}{2}\int_0^t\int_{Q_d}\rho^{\theta}(s)\left(\frac{f(\rho^{\theta}(s))}{\rho^{\theta}(s)}-\frac{f(\rho(s))}{\rho(s)}\right)^2dsdx+\int_0^t|D_1(s)|ds.
    \end{align*}
    We choose $\eta=\nu/2$ and get
    \begin{align}\label{eq:burgers-numerical-proof-5}
        &\ca{H}(\rho^{\theta}(t)|\rho(t))-\ca{H}(\rho^{\theta}(0)|\rho(0))+\frac{\nu}{2}\int_0^t\int_{Q_d}\rho^{\theta}(s)\left(\frac{\partial_x\rho^{\theta}(s)}{\rho^{\theta}(s)}-\frac{\partial_x\rho(s)}{\rho(s)}\right)^2dxds\nonumber\\
        &\le\nu^{-1}\int_0^t\int_{Q_d}|g|^2dsdx+\nu^{-1}\int_0^t\int_{Q_d}\rho^{\theta}(s)\left(\frac{f(\rho^{\theta}(s))}{\rho^{\theta}(s)}-\frac{f(\rho(s))}{\rho(s)}\right)^2dsdx+\int_0^t|D_1(s)|ds.
    \end{align}
    Given the definition of $D_1$ in \eqref{eq:burgers-numerical-proof-D1}, the last term in \eqref{eq:burgers-numerical-proof-5} satisfies for all $t\in[0,T]$
    \begin{align}\label{eq:burgers-numerical-proof-6}
        \int_0^t&|D_1(s)|ds\le C\ell_T(\log(\rho^{\theta}))^{1/2}\left(\nu^2\|\partial_{xx}\rho^{\theta}\|^2_{L^2([0,T]\times[0,1])}+\|r^{\theta}(\cdot,1)\|^2_{L^2([0,T])}+1\right)^{1/2}\nonumber\\
        &+C\left(\ell_T(\rho^{\theta})^{1/2}+\nu\ell_T(\partial_x\rho^{\theta})^{1/2}+\ell_T(r^{\theta})^{1/2}\right)\nonumber\\
        &\times\left(\|\partial_x\log\rho^{\theta}\|_{L^2([0,T]\times[0,1])}+\|\partial_{x}\log\rho\|^2_{L^2([0,T]\times[0,1])}+1\right)^{1/2}\nonumber\\
        &+C\delta^{-1}\ell_T(\rho^{\theta})^{1/2}\left(\nu^2\|\partial^2_{xx}\rho\|_{L^2([0,T]\times[0,1])}+1\right)^{1/2},
    \end{align}
    where $C>0$ is an independent constant. The analogous of \eqref{eq:burgers-numerical-proof-5} and \eqref{eq:burgers-numerical-proof-6} for \eqref{eq:burgers-numerical-proof-3} are obtained by the same arguments. Next, we collect the results for the time derivative of $\ca{H}(\rho^{\theta}(t)|\rho(t))$ and $\ca{H}(1-\rho^{\theta}(t)|1-\rho(t))$ and arrive at
    \begin{align}\label{eq:burgers-numerical-proof-7}
        &\ca{H}(\rho^{\theta}(t)|\rho(t))+\ca{H}(1-\rho^{\theta}(t)|1-\rho(t))+\nu\int_0^t\left(\ca{F}(\rho^{\theta}(s)|\rho(s))+\ca{F}(1-\rho^{\theta}(s)|1-\rho(s))\right)ds\nonumber\\
        &=\ca{H}(\rho^{\theta}(0)|\rho(0))+\ca{H}(1-\rho^{\theta}(0)|1-\rho(0))+2\nu^{-1}\left(I_{\text{dyn}}(\theta)+\varepsilon\right)+\int_0^T\left(|D_1(s)|+|D_2(s)|\right)ds\nonumber\\
        &+\nu^{-1}\int_0^T\int_{Q_d}|\rho^{\theta}(s)-\rho(s)|^2dxds,
    \end{align}
    where we used the definition of Fisher information \eqref{eq:fisher}, and \eqref{eq:burgers-numerical-proof-1.5}. The last term in \eqref{eq:burgers-numerical-proof-7}, can be dealt by the same arguments as in \cite{ryck-mishra24}. This produces the following estimate
    \begin{align}\label{eq:burgers-numerical-proof-8}
        \int_0^T\int_{Q_d}|\rho^{\theta}(s)-\rho(s)|^2dxds \le C\left(\int_{Q_d}|\rho^{\theta}(0)-\rho(0)|^2dx+I_{\text{dyn}}(\theta)+B\right)\exp\left(T\nu^{-1}[f]_L+1\right),
    \end{align}
    where
    \begin{align*}
        B=&C\left(\|\partial_{xx}^2\rho^{\theta}\|_{L^2([0,T]\times[0,1])}+\|\partial_{xx}^2\rho\|_{L^2([0,T]\times[0,1])}\right)^{1/2}\ell_T(\rho^{\theta})^{1/2}+T^{1/2}\ell_T(\rho^{\theta})^{1/2}\\
        &+2T^{1/2}(1+[f]_L)\ell_T(\rho^{\theta})^{1/2}+T^{1/2}\|r^{\theta}(\cdot,1)\|_{L^2([0,T])}\ell_T(\rho^{\theta})^{1/2}+2T^{1/2}\ell_T(r^{\theta})^{1/2}.
    \end{align*}
    By using \eqref{eq:burgers-numerical-proof-8} in \eqref{eq:burgers-numerical-proof-7}, and taking $\varepsilon\to 0$, the desired estimate follows.
\end{proof}
}
\subsection{Navier-Stokes equation}

The incompressible Navier-Stokes equations are written as
\begin{align}\label{eq:nse-convergence}
    \partial_t u-\nu\Delta u + \nabla\cdot(u\otimes u) + \nabla p&=0,\quad 
    \nabla\cdot u=0,
\end{align}
with initial condition $u(0)=u_0$,
where $u\colon[0,T]\times\Td\to\Rd$ denotes the velocity field, $p\colon[0,T]\times\Td\to\R$ represents the pressure, and $u_0$ is a prescribed initial condition. Given two networks models $u^{\theta}$ and $\rho^{\theta}$, let us denote the residual of the equation by
\begin{align*}
    R(\theta)=\partial_t u^{\theta}-\nu\Delta u^{\theta} + \nabla\cdot(u^{\theta}\otimes u^{\theta}) + \nabla p^{\theta}.
\end{align*}
    
\begin{theorem}\label{thm:nse-generalization}
	Let $\psi^{\theta}\colon[0,T]\times\Rd\to\R$ and $p^{\theta}\colon[0,T]\times\Rd\to\R$ be two neural networks, and define $u^{\theta}$ as in \eqref{eq:velocity-nse}. Then there exists a constant $\kappa=\kappa\left(d,\nu,T,\|u\|_{L^2([0,T];H^2)},\|u^{\theta}\|_{L^2([0,T];H^2)}\right),$ such that for all $\theta$ we have
    \begin{align*}
    \sup_{t\in[0,T]}\int_{\Td}|u(t)-u^{\theta}(t)|^2dx&+\nu\int_0^T\int_{\Td}|\nabla u(s)-\nabla u^{\theta}(s)|^2dxds\\
    &\le\kappa\left(\ca{L}_{\text{bc}}^{1/2}(u^{\theta})+\ca{L}_{\text{bc}}^{1/2}(\nabla u^{\theta})+I_{0}(\theta)+I_{\text{dyn}}(\theta)\right).
    \end{align*}
\end{theorem}
\begin{proof}
    Since the principal difference in our setting lies in the computation of the rate functional, we restrict ourselves to highlighting the essential difference with the proof presented in \cite{ryck-jagtap-mishra23}.  
    
    The numerical analysis of PINNs for the Navier--Stokes equations (NSE) on periodic domains has been carried out in \cite{ryck-jagtap-mishra23}. The central idea in that work is to study the evolution of the $L^2$-distance $\|u - u^{\theta}\|_{L^2}$, and to proceed along standard arguments for the stability of solutions to NSE. In contrast, in the present framework the residual $R(\theta)$ takes values in $L^2_TH^{-1}$. Consequently, multiplication by $u - u^{\theta}$ must be interpreted as evaluating the distribution $R(\theta)$ against the test function $u - u^{\theta}$, and the viscous term is then employed to absorb the resulting $\|u - u^{\theta}\|_{H^1}$ contribution.
\end{proof}

\section{Numerical results}\label{sec:numerical-results}

We assess the performance of the proposed method on two prototypical physical systems. As a first case study, we consider the viscous Burgers’ equation, with emphasis on the small-viscosity regime where shock formation occurs. This regime is of particular interest, since it is well-documented that PINNs exhibit difficulties in accurately capturing solutions with sharp gradients. As a second case study, we investigate the incompressible two-dimensional Navier--Stokes equations.  

The loss functional is minimized using the full-batch Adam optimizer with an exponentially decaying learning rate. Network parameters are initialized according to the Glorot scheme \cite{glorot10}. Apart from the resampling of collocation points, employed to mitigate overfitting, only standard hyperparameter choices and training routines are used. All neural networks are taken to be fully connected, with hyperbolic tangent activation functions.  

During training, we additionally monitor the evolution of the approximated rate functional, either in the form \eqref{eq:dyn-rate-numerical} for the Burgers’ equation, or in the form \eqref{eq:nse-rate} for the Navier--Stokes system.

To ensure robustness with respect to parameter initialization, each experiment is repeated multiple times with independently drawn initial weights. Further details of the experimental setup, including network architectures, optimizer parameters, and training schedules, are provided in Section~\ref{sec:experimental-setup}.

Following the stability estimates established in Section~\ref{sec:analysis}, the common practice in the literature \cite{hao23}, and the large deviation principle \eqref{eq:ldp-ssep}, we adopt as evaluation metrics the $L^p$-relative errors for $p=1,2$, defined between the approximation $\rho^{\theta}$ and a reference solution $\bar{\rho}$ by  
\begin{align}\label{eq:2-rel-err}
    \epsilon_p(\theta) 
    = \frac{\|\rho^{\theta}-\bar{\rho}\|_{L^p_{t,x}}}{\|\bar{\rho}\|_{L^p_{t,x}}} \times 100.
\end{align}
In addition, we record the value of the approximated rate functional $I$. Both quantities are evaluated on an independent set of collocation points, disjoint from the training set.  

The numerical results indicate a relative $L^2$ error of order $\sim 10\%$, which is competitive with existing benchmarks for PINNs reported in \cite{hao23}.

Finally, to ensure a comparability with alternative formulations of the loss, and to render the computation of the rate functional directly accessible, all methods under consideration are implemented using the same network architecture as employed for THINNs.

\subsection{Experimental setup}\label{sec:experimental-setup}
For the training of the networks, we employ full-batch Adam gradient descent. The initial learning rate is denoted by $\eta_0 \in (0,1)$, and is multiplied by a decay factor $\gamma \in (0,1)$ after a prescribed number of steps. The total number of training steps is denoted by $S \in \mathbb{N}$.  

As discussed above, the algorithm is repeated $N>1$ times with independently drawn initializations, yielding a collection of values $\epsilon_2^{ij}$, where $i=1,\ldots,N$ indexes the run and $j=1,\ldots,K$ indexes the training step. At each step $j$, we compute the median and the first quantile with respect to $i$, and report the resulting sequences. An identical procedure is applied to the values of the rate functional $I$.  

The initial conditions considered for the Burgers’ equation are displayed in Figure~\ref{fig:init-conds}. The initial profile used for the Navier--Stokes case will be specified lateron. In the case of Burgers’ equation, to prevent instabilities of the finite-difference scheme, the initial profile is rescaled accordingly.

\begin{figure}[ht]
    \centering
    \includegraphics[width=0.7\linewidth]{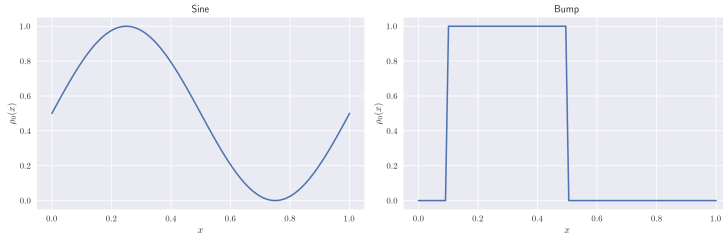}
    \caption{Initial conditions for the Burgers' equation, satisfying periodic boundary conditions.}
    \label{fig:init-conds}
\end{figure}

\subsection{Viscous Burgers' equation}
We first present numerical results for the one-dimensional viscous Burgers’ equation, given by  
\begin{align}\label{eq:burgers-rho-nu}
    \partial_t \rho = \nu \, \partial_{xx}^2 \rho - \partial_x\!\big(\rho(1-\rho)\big),
\end{align}
where $\nu>0$ denotes the viscosity parameter controlling the smoothing effect of the diffusive term. This nonlinear conservation law is widely employed as a benchmark for PINN algorithms, since it is relatively simple yet already exhibits rich dynamical features such as approximate shock formation even in one dimension \cite{li24,ryck-mishra-molinaro24,mishra-molinaro22,raissi19,chaumet-giesselman24}.  

From the perspective of statistical physics, \eqref{eq:burgers-rho-nu} arises as the hydrodynamic limit of the weakly asymmetric exclusion process (WASEP) \cite{kipnis89}. In this case, the thermodynamic gradient-flow structure retains the same mobility as for the symmetric exclusion process (SSEP), namely $\Phi(\rho) = \rho(1-\rho)$, while the energy functional corresponds to the mixing entropy
\begin{align}
    \mathcal{E}^{\mathrm{MixEnt}}[\rho] =\int_{\Td}\rho\log\rho dx + \int_{\Td}(1-\rho)\log(1-\rho) dx,
\end{align}
augmented by a correction term accounting for the nonlinear transport:  
\begin{align}\label{eq:burgers-energy-mobility}
    \mathcal{K}^{\mathrm{B}}_{\rho}\varphi = -\partial_x\!\big(\rho(1-\rho)\,\partial_x\varphi\big), 
    \qquad 
    \mathcal{E}^{\mathrm{B}}[\rho] = \nu\,\mathcal{E}^{\mathrm{MixEnt}}[\rho] - \int x \rho \, dx.
\end{align}
Equivalently, the tangent space satisfies $L^2_TH^{-1}_{\Phi}$ with $\Phi(t,x)=\rho(t,x)(1-\rho(t,x))$.

As already discussed, a particularly interesting regime is the hyperbolic limit $\nu \to 0$, in which solutions of \eqref{eq:burgers-rho-nu} develop discontinuities in finite time, known as shocks. It has been observed in the literature (see, e.g., \cite{mishra-molinaro22,ryck-mishra-molinaro24,chaumet-giesselman24}), and is also confirmed in the present study, that PINNs fail to provide accurate approximations in this regime when $\nu \ll 1$.  

To evaluate the performance of THINNs in the small-viscosity regime, we consider two types of initial data: a sine wave and a bump profile, both displayed in Figure~\ref{fig:init-conds}. These profiles are rescaled to take values in $[0,1]$ and to satisfy the prescribed boundary conditions. Furthermore, we investigate the role of the thermodynamically informed weighting in the loss functional by comparing the proposed approach with the classical choice based on $L^2$-penalization of the residual.

The network architecture used for $w^{\theta} \in \mathcal{W}$ is a fully connected feedforward neural network with six hidden layers, each consisting of $64$ neurons. The evolution of the validation metrics $\epsilon_1$ and $\epsilon_2$ is reported in Figures~\ref{fig:wrapped_burgers_sine_nu1e-5_q25} and \ref{fig:wrapped_burgers_bump_nu1e-5_q25} for equation~\eqref{eq:burgers-rho-nu} with sine and bump initial conditions, respectively.  

The thermodynamically informed choice of residual penalization based on the norm of $L^2_TH^{-1}_{\Phi}$, significantly outperform standard PINNs with respect to the $L^1$- and $L^2$-relative errors. Moreover, the value of the rate functional obtained with THINNs is at least four order of magnitude smaller, indicating that the corresponding outputs are considerably more likely to represent realizations of the underlying physical system.

\begin{remark}
    To obtain reference solutions, we employ a finite-difference scheme. It was observed that the associated CFL condition becomes increasingly restrictive as $\nu \to 0$, necessitating the use of progressively finer meshes. To alleviate this difficulty, a simple rescaling of the initial conditions was applied. We emphasize that this restriction is absent in the neural network approaches, which are mesh-free by construction.
\end{remark}
\begin{figure}[ht]
    \centering
    \includegraphics[width=1\linewidth]{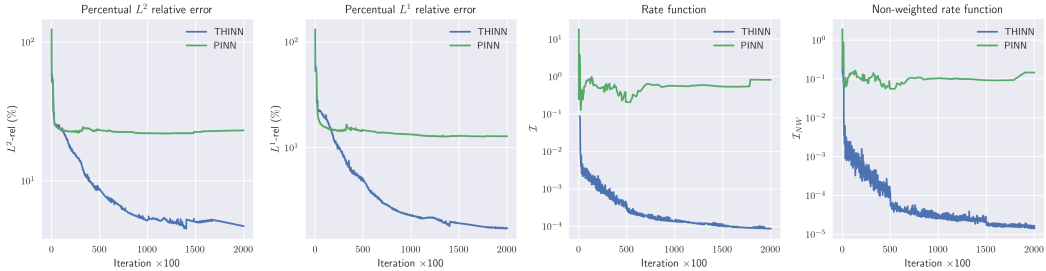}
    \caption{
    Evaluation metrics as a function of the gradient descent iterations for equation~\eqref{eq:burgers-rho-nu} with sine initial condition and viscosity $\nu = 10^{-5}$. The curves display the evolution of the median over $15$ independent runs. The experiment was performed with $S = 200{,}000$ training steps and initial learning rate $\eta_0 = 10^{-4}$ (see Section~\ref{sec:experimental-setup}).}
    \label{fig:wrapped_burgers_sine_nu1e-5_q25}
\end{figure}

\begin{figure}[ht]
    \centering
    \includegraphics[width=1\linewidth]{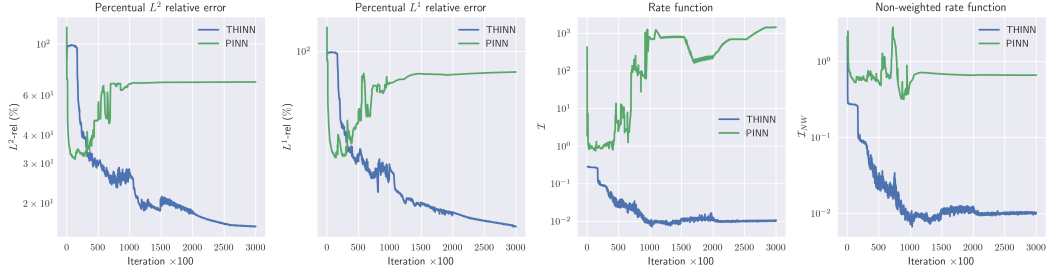}
    \caption{
    Evaluation metrics as a function of the gradient descent iterations for equation~\eqref{eq:burgers-rho-nu} with bump initial condition and viscosity $\nu = 10^{-5}$. The curves display the evolution of the first quantile over $11$ independent runs. The experiment was performed with $S = 300{,}000$ training steps, decay step $k = 100{,}000$, initial learning rate $\eta_0 = 10^{-4}$, and decay factor $\gamma = 0.2$ (see Section~\ref{sec:experimental-setup}).}
    \label{fig:wrapped_burgers_bump_nu1e-5_q25}
\end{figure}

\subsection{Navier-Stokes equation}
We apply the method proposed in  Section \ref{sec:numerical-implementation-nse} to approximate the solution of the two dimensional incompressible Navier-Stokes equation
\begin{align}\label{eq:nse}
    \partial_t u-\nu\Delta u + \nabla\cdot(u\otimes u) + \nabla p&=0,\quad 
    \nabla\cdot u=0,
\end{align}
where the differential operators are applied row-wise on matrices. For the underlying physical system, and derivation of the rate function we refer to \cite{quastel98,gess-heydecker-wu24}. The rate function is then given by the non-weighted dual norm of the residual, i.e., $L^2_TH^{-1}$. We choose as a testing instance the Taylor-Green vortex in two dimensions with exact solution given by
\begin{align*}
	u &= \sin(x)\cos(y)e^{-2\nu t},\\
	v &= -\cos(x)\sin(y)e^{-2\nu t}.
\end{align*}
The evolutions of the evaluation metrics are displayed in Figure \ref{fig:nse1}. The thermodynamically informed choice of residual penalization based on the norm of $L^2_TH^{-1}$, significantly outperforms standard PINNs with respect to the $L^1$- and $L^2$-relative errors. Moreover, the value of the rate functional obtained with THINNs is of approximated one order of magnitude smaller, indicating that the corresponding outputs are more likely to represent realizations of the underlying physical system.
\begin{figure}[ht]
    \centering
    \includegraphics[width=1\linewidth]{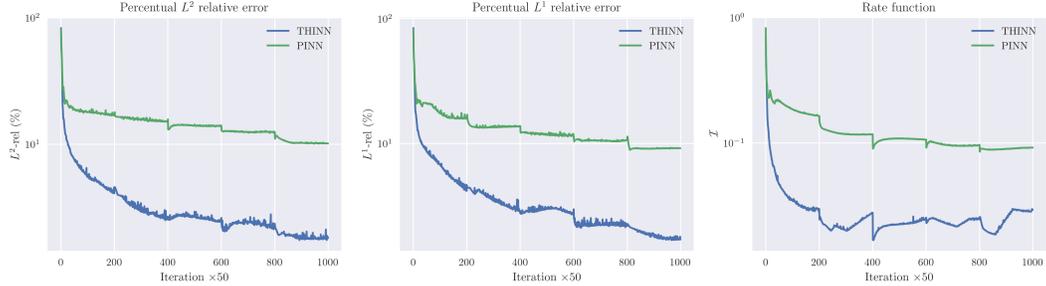}
    \caption{Evaluation metrics as a function of the gradient descent iterations for equation~\eqref{eq:nse}. For this example, $\nu=0.5$ and $T=5$. For the training: $S=50{,}000$, the initial learning rate was $\eta_0=10^{-3}$ which was reduced during training, and the collocation points for the physical loss were resampled every $10{,}000$ iterations.}
    \label{fig:nse1}
\end{figure}
\section{Conclusions}

Physics-informed neural networks (PINNs) incorporate physical knowledge by training under the constraint of satisfying the physical laws encoded in the governing PDE. In the context of fluctuating nonequilibrium thermodynamics this philosophy is extended in this paper by introducing a novel framework, termed \emph{thermodynamically consistent PINNs} (THINNs), in which also the penalization of the residual from the governing PDE is chosen based on physical principles. 

The method is shown to be consistent to the perspective of gradient flow interpretations of PDEs, as well as to macroscopic fluctuation theory. It is further argued that the resulting minimization of the loss function corresponds to a penalization of the discretization error that is consistent with the modelling error. 

Subsequently, a-posteriori error estimates for the THINN method are shown. These bound the error from solutions to PDEs in terms of the training loss. Particular care was required in the treatment of boundary conditions. Whereas the derivation of such bounds in the standard PINN framework relies on classical a priori PDE estimates, the analysis for THINNs proceeds by studying the evolution of the relative entropy, that is, in a consistent form to the underlying rate functions.

Finally, our numerical experiments demonstrate the advantages of THINNs. In the case of Burgers’ equation with small viscosity, THINNs outperform PINNs in terms of $L^1$- and $L^2$-relative errors, for both initial conditions considered. Moreover, for the Burgers’ and Navier--Stokes equations, THINNs consistently achieve smaller rate functional values compared to their counterpart, confirming that the THINNs output corresponds to more probable realizations of the underlying physical system.  

\appendix
\section{Appendix} 

This appendix contains supplementary material complementing the main text.  
Section~\ref{sec:weighted-sobolev} introduces the weighted Sobolev space $L^2H^1_{\Phi}$, which plays a central role in the formulation of the rate functional.  
Section~\ref{sec:GF-IPS} establishes the formal connection between a general interacting particle system $\pi^N$ and the gradient-flow structure of its hydrodynamic limit.  
Section~\ref{sec:relative_entropy} provides background on relative entropy and relative Fisher information, including auxiliary results used in the stability analysis.

\subsection{Weighted Sobolev spaces}\label{sec:weighted-sobolev}
Let $U \in \{\mathbb{T}^d, Q_d\}$ with $Q_d = [0,1]^d$ and $T>0$, and let $\Phi \colon [0,T]\times U \to [0,\infty)$ be a measurable function. Define a equivalence relation $\sim$ on smooth functions $u,v$ by  
\[
    u \sim v 
    \;\;\Longleftrightarrow\;\;
    \int_0^T \int_U |\nabla u(t) - \nabla v(t)|^2 \,\Phi(t) \, dx \, dt = 0.
\]
Unless otherwise specified, spatial integrals will always be taken over $U$.  

We consider the quotient space $C^{1,2}([0,T]\times U)/\sim$ and endow it with the inner product  
\begin{align}\label{eq:h1-weighted-inner-product}
    \langle u, v \rangle_{\Phi}
    = \int_0^T \int_U (\nabla u(t) \cdot \nabla v(t)) \,\Phi(t)\, dx \, dt,
\end{align}
which induces the norm $\|u\|_{\Phi}^2 = \langle u,u \rangle_{\Phi}$. Note that spatially constant functions are identified with the zero class.  

The associated Sobolev space is denoted by $L_T^2H^1_{\Phi}$ and is defined as the Hilbert space obtained by completing the quotient space with respect to $\|\cdot\|_{\Phi}$, i.e.,  
\begin{align}\label{eq:h1-def}
    L_T^2H^{1}_{\Phi}
    := \overline{C^{1,2}([0,T]\times U)/\sim}^{\|\cdot\|_{\Phi}},
\end{align}
where in the case $U = \mathbb{T}^d$, functions are understood to be periodic. The topological dual is denoted by $L_TH^{-1}_{\Phi}$. For a distribution $f \in L_T H^{-1}_{\Phi}$, its action on test functions $\varphi \in L_T H^1_{\Phi}$ is written as $\langle f,\varphi \rangle$, and its norm as $\|f\|_{\Phi,*}$.

\subsection{Gradient flows and particle systems}\label{sec:GF-IPS}
In this section we introduce the Riemannian structure underlying our analysis and relate it to the Sobolev space defined in the previous section. Let $M$ denote the space of positive, absolutely continuous probability measures on $\mathbb{T}^d$ with densities bounded by one. We equip $M$ with the weak topology, and denote by $D([0,T],M)$ the space of càdlàg paths in $M$, endowed with the Skorohod topology.

We now recall the LDP satisfied by the symmetric simple exclusion process (SSEP), following \cite{kipnis89,kipnis99}, and introduce its associated rate functional. Let $\pi^N$ denote the empirical measure of the system, as defined in the introduction, and set $\Phi(x) = x(1-x)$. For each smooth test function $H \in C^{1,2}([0,T]\times \mathbb{T}^d)$, define the functional $J_H \colon D([0,T],M) \to \mathbb{R}$ by
\[
    J_H(\rho) = l_H(\rho) - \frac{1}{2} \int_0^T \int_{\mathbb{T}^d} 
        \rho_t(1-\rho_t)\,|\nabla H_t|^2 \, dx \, dt,
\]
where $l_H \colon D([0,T],M) \to \mathbb{R}$ is given by
\[
    l_H(\rho) = \int_{\mathbb{T}^d} \rho_T H_T \, dx
        - \int_{\mathbb{T}^d} \rho_0 H_0 \, dx
        - \int_0^T \int_{\mathbb{T}^d} 
            \big(\partial_t H_t + \tfrac{1}{2}\Delta H_t\big) \rho_t \, dx \, dt.
\]

We focus for a moment on the LDP satisfied by the SSEP \cite{kipnis99, kipnis89}, we seek to rigorously introduce its rate function. Let $\pi^N$ be the empirical measure of such system as defined in the introduction section, and $\Phi(x)=x(1-x)$. For each smooth $H\in C^{1,2}([0,T]\times\Td)$, let $J_H\colon D([0,T],M)\to\R$ be defined by
\begin{align*}
    J_H(\rho)=l_H(\rho)-\frac{1}{2}\int_0^T\int_{\Td}\rho_t(1-\rho_t)|\nabla H_t|^2dxdt,
\end{align*}
where the function $l_H\colon D([0,T],M)\to \R$ is defined by
\begin{align*}
    l_H(\rho)=\int_{\Td}\rho_T H_Tdx-\int_{\Td}\rho_0 H_0dx-\int_0^T\int_{\Td}(\partial_t H + \frac{1}{2}\Delta H_t)\rho_t dxdt.
\end{align*}
The rigorous formulation of the dynamical rate functional for the SSEP is given by  
\begin{align}\label{eq:rate-sup-form}
    \mathcal{I}_{\mathrm{dyn}}(\rho)
    = \sup\Big\{ J_H(\rho) \;:\; H \in C^{1,2}([0,T]\times\mathbb{T}^d) \Big\}.
\end{align}
It can be shown that if $\rho \in D([0,T],M)$ satisfies $\mathcal{I}_{\mathrm{dyn}}(\rho) < \infty$, then there exists a unique $H \in L^2_TH^1_{\Phi \circ \rho}$ such that $\rho$ solves
\begin{align}\label{eq:ssep-rate-rigorous}
    \partial_t \rho - \Delta \rho
    = - \nabla \cdot \big( \Phi(\rho) \nabla H \big),
    \qquad
    \mathcal{I}_{\mathrm{dyn}}(\rho)
    = \|H\|_{L^2_TH^1_{\Phi \circ \rho}}^2.
\end{align}
The function $H$ is obtained via the Riesz representation theorem: equation \eqref{eq:ssep-rate-rigorous} expresses that the residual $(\partial_t - \Delta)\rho \in L^2_TH^{-1}_{\Phi \circ \rho}$ is represented by $H$, while the norm of $H$ coincides with the value of the rate functional. Equivalently,
\begin{align}\label{eq:ssep-rate-rigorous-2}
    \mathcal{I}_{\mathrm{dyn}}(\rho)
    = \int_0^T \int_{\mathbb{T}^d} \Phi(\rho)\,|\nabla H|^2 \, dx \, dt
    = \|(\partial_t - \Delta)\rho\|_{L_T^2H^{-1}_{\Phi \circ \rho}}^2.
\end{align}
The chain of equalities \eqref{eq:ssep-rate-rigorous-2} highlights the connection between the rate functional and a natural gradient-flow structure. In what follows we will consider generic interacting particle systems whose rate functionals admit such a representation.

Let $\pi^N$ be a $D([0,T],M)$-valued random variable representing the empirical measure of a fixed interacting particle system, and assume that it satisfies a large deviation principle with a dynamical rate functional $\mathcal{I}_{\mathrm{dyn}} \colon M \to \overline{\mathbb{R}}$ of the form
\begin{align}\label{eq:GF-IPS-rate}
    \mathcal{I}_{\mathrm{dyn}}(\bar\rho)
    = \frac{1}{2} \int_0^T 
        \Big\| \partial_t \bar\rho 
        + \mathcal{K}_{\rho}\frac{\delta \mathcal{E}}{\delta\rho}[\bar\rho] 
        \Big\|_{L^2_TH^{-1}_{\Phi\circ\rho}}^2 \, dt,
\end{align}
where $\Phi \colon [0,1]\to(0,1)$ is a measurable mobility function, $L^2_TH^{-1}_{\Phi\circ\rho}$ is the weighted Sobolev space introduced in Section~\ref{sec:weighted-sobolev}, and the operator $\mathcal{K}_{\rho} \colon L^2_TH^1_{\Phi\circ\rho} \to L^2_TH^{-1}_{\Phi\circ\rho}$ is defined by
\begin{align}\label{eq:K}
    \mathcal{K}_{\rho}\psi = - \nabla \cdot \big(\Phi(\rho)\nabla\psi\big).
\end{align}
Equation \eqref{eq:K} is understood in the weak sense, namely, for all $\varphi \in L^2H^1_{\Phi\circ\rho}$,
\[
    \langle \mathcal{K}_{\rho}\psi, \varphi \rangle
    = \int_0^T \int_{\mathbb{T}^d} \Phi(\rho)\,\nabla \psi \cdot \nabla \varphi \, dx \, dt.
\]

We further assume that the energy functional $\mathcal{E} \colon M \to \mathbb{R}$ admits a distributional derivative of the form
\[
    \frac{\delta \mathcal{E}}{\delta\rho}[\rho] = f(\rho),
\]
for some measurable function $f \colon (0,1)\to\mathbb{R}$. Introducing a formal Riemannian structure on $M$ by attaching to each $\rho \in M$ the tangent space $T_\rho M := L^2H^{-1}_{\Phi\circ\rho}(\mathbb{T}^d)$, the functional \eqref{eq:GF-IPS-rate} can be reformulated as
\begin{align}\label{eq:GF-IPS-rate-2}
    \mathcal{I}_{\mathrm{dyn}}(\bar\rho)
    &= \frac{1}{2}\int_0^T 
        \Big\| \partial_t \bar\rho 
            + \nabla\cdot\big(\Phi(\bar\rho)\nabla f[\bar\rho]\big)
        \Big\|_{T_{\bar\rho}M}^2 \, dt \nonumber\\
    &= \frac{1}{2}\inf\Bigg\{
        \int_0^T \int_{\mathbb{T}^d} |g|^2 \, dx \, dt \;:\;
        \partial_t \bar\rho 
        + \mathcal{K}_{\rho}\frac{\delta\mathcal{E}}{\delta\rho}[\bar\rho]
        = -\nabla \cdot \big(\sqrt{\Phi(\bar\rho)}\,g\big)
    \Bigg\}.
\end{align}

The equation appearing in the infimum of \eqref{eq:GF-IPS-rate-2} is commonly referred to as the {skeleton equation}, and implicitly carries periodic boundary conditions. A function $g \in L^2([0,T]\times\mathbb{T}^d)$ is called a weak solution of the skeleton equation if, for all $\varphi \in C^\infty([0,T]\times\mathbb{T}^d)$,
\begin{align}\label{eq:skeleton-weak-sense}
    \int_0^T \int_{\mathbb{T}^d} \partial_t \bar\rho\,\varphi \, dx\,dt
    - \int_0^T \int_{\mathbb{T}^d} 
        \Phi(\bar\rho)\,\nabla f(\bar\rho)\cdot \nabla \varphi \, dx\,dt
    = \int_0^T \int_{\mathbb{T}^d} 
        \Phi^{1/2}(\bar\rho)\,g \cdot \nabla \varphi \, dx\,dt.
\end{align}
We refer to \cite{kipnis89,dirr-fehrman-gess24,gess-heydecker25} for a detailed analysis of the skeleton equation associated with the zero-range process and the SSEP.

In this work we focus on three prototypical PDEs: the heat equation, the viscous Burgers equation, and the incompressible Navier--Stokes equations. Each of these arises, as discussed above, as the hydrodynamic limit of a different interacting particle system, namely the SSEP \cite{kipnis99}, the WASEP \cite{kipnis89}, and a stochastic lattice gas model \cite{quastel98}, respectively. We also note that the fluctuating structure found for Navier--Stokes is consistent with those found in macroscopic fluctuation theory, see \cite{oettinger05,gess-heydecker-wu24}. In each case the associated mobility operator $\Phi$ determines the geometry of the tangent space, and the corresponding rate functional can be expressed, up to a constant, as the time integral of the PDE residual measured in a negative $\Phi$-weighted Sobolev space (see Section~\ref{sec:weighted-sobolev}).

\subsection{Relative entropy and relative Fisher information}\label{sec:relative_entropy}
For the cases of the heat and Burgers’ equations, the static rate functional is given by the relative entropy. We therefore briefly recall this quantity and an important result linking it to the $L^1$-distance.  

Given two measurable functions $f,g \colon Q_d \to \mathbb{R}_+$ with $g \geq \delta > 0$, the relative entropy of $f$ with respect to $g$ is defined by
\[
    \mathcal{H}(f\,|\,g) = \int_{Q_d} f(x) \log\!\left(\frac{f(x)}{g(x)}\right) dx.
\]
This functional quantifies the discrepancy between the densities $f$ and $g$, although it does not define a distance on the space of finite measures. With this definition, the static large-deviation functional associated with an initial condition $\rho_0$ can be written as
\begin{equation}\label{eq:rel_entr_symm}
    \mathcal{I}_0(\rho) = \mathcal{H}(\rho\,|\,\rho_0) + \mathcal{H}(1-\rho\,|\,1-\rho_0).
\end{equation}
Similar to \cite{yoshida17}, we study the evolution of the relative entropy between the approximation and the true solution. In these computations, the relative Fisher information naturally appears, defined by
\begin{align}\label{eq:fisher}
    \mathcal{F}(f\,|\,g) = \int_{Q_d} f(x)\,
        \Bigg|\nabla \log\!\left(\frac{f(x)}{g(x)}\right)\Bigg|^2 dx.
\end{align}

The following classical estimate, known as Pinsker’s inequality, will be used to express stability in terms of the $L^1$-distance; see \cite{yoshida17} and the references therein.  

\begin{theorem}[Pinsker’s inequality]\label{thm:l1-entropy}
    Let $\mu$ and $\nu$ be two absolutely continuous Borel probability measures on $Q_d=[0,1]^d$, with densities $f$ and $g$, respectively. Then
    \begin{align}\label{eq:pinsker}
        \Bigg(\int_{Q_d} |f(x)-g(x)|\,dx\Bigg)^2
        \;\leq\; \mathcal{H}(f\,|\,g).
    \end{align}
\end{theorem}
For the sake of notational convenience, we introduce the symmetric relative entropy and relative Fisher information given by
\begin{align}\label{eq:symmetric-ent}
    \ca{H}^s(f|g)&=\ca{H}(f|g)+\ca{H}(1-f|1-g),
\end{align}
and
\begin{align}\label{eq:symmetric-fish}
    \ca{F}^s(f|g)&=\ca{F}(f|g)+\ca{F}(1-f|1-g),
\end{align}
respectively. With this notation, we have that $\ca{I}_0(\rho)=\ca{H}^s(\rho|\rho_0)$.
\section*{Acknowledgments} From October 2022 through January 2025, JC was funded by the Deutsche Forschungsgemeinschaft (DFG, German Research Foundation) – IRTG 2235- Project number 282638148.  BG acknowledges
support from DFG CRC/TRR 388 “Rough Analysis, Stochastic Dynamics and Related Fields”, Project A11.
\bibliographystyle{plain}
\bibliography{bibliography}

\begin{thebibliography}{10}

\bibitem{adams-dirr-peletier-zimmer11}
Stefan Adams, Nicolas Dirr, Mark~A. Peletier, and Johannes Zimmer.
\newblock From a large-deviations principle to the wasserstein gradient flow: A new micro-macro passage.
\newblock {\em Communications in Mathematical Physics}, 307(3):791--815, 2011.

\bibitem{adams-dirr-peletier-zimmer13}
Peletier~Mark Adams~Stefan, Dirr~Nicolas and Zimmer Johannes.
\newblock Large deviations and gradient flows.
\newblock {\em Phil. Trans. R. Soc. A.}, 371, 2013.

\bibitem{beania-sadallah16}
Yassine Benia and Boubaker-Khaled Sadallah.
\newblock Existence of solutions to burgers equations in domains that can be transformed into rectangles.
\newblock {\em Electronic Journal of Differential Equations}, 2016:1--13, 06 2016.

\bibitem{bonito24}
Andrea Bonito, Ronald DeVore, Guergana Petrova, and Jonathan~W. Siegel.
\newblock Convergence and error control of consistent pinns for elliptic pdes, 2024.

\bibitem{Cai21}
Shengze Cai, Zhiping Mao, Zhicheng Wang, Minglang Yin, and George~Em Karniadakis.
\newblock Physics-informed neural networks (pinns) for fluid mechanics: A review.
\newblock {\em Acta Mechanica Sinica}, 37(12):1--24, 2021.

\bibitem{chaumet-giesselman24}
Aidan Chaumet and Jan Giesselmann.
\newblock Improving {Weak} {PINNs} for {Hyperbolic} {Conservation} {Laws:} {Dual} {Norm} {Computation,} {Boundary} {Conditions} and {Systems}.
\newblock {\em The SMAI Journal of computational mathematics}, 10:373--401, 2024.

\bibitem{chuang-22}
Pi-Yueh Chuang and Lorena~A. Barba.
\newblock {Experience Report of Physics-Informed Neural Networks in Fluid Simulations: Pitfalls and Frustration}, 10 2022.

\bibitem{ryck-jagtap-mishra23}
Tim De~Ryck, Ameya~D Jagtap, and Siddhartha Mishra.
\newblock Error estimates for physics-informed neural networks approximating the navier–stokes equations.
\newblock {\em IMA Journal of Numerical Analysis}, 44(1):83--119, 01 2023.

\bibitem{ryck-mishra-molinaro24}
Tim De~Ryck, Siddhartha Mishra, and Roberto Molinaro.
\newblock wpinns: Weak physics informed neural networks for approximating entropy solutions of hyperbolic conservation laws.
\newblock {\em SIAM Journal on Numerical Analysis}, 62(2):811--841, 2024.

\bibitem{dietrich23}
Felix Dietrich, Alexei Makeev, George Kevrekidis, Nikolaos Evangelou, Tom Bertalan, Sebastian Reich, and Ioannis~G. Kevrekidis.
\newblock Learning effective stochastic differential equations from microscopic simulations: Linking stochastic numerics to deep learning.
\newblock {\em Chaos: An Interdisciplinary Journal of Nonlinear Science}, 33(2):023121, 02 2023.

\bibitem{dirr-fehrman-gess24}
Nicolas Dirr, Benjamin Fehrman, and Benjamin Gess.
\newblock Conservative stochastic pde and fluctuations of the symmetric simple exclusion process, 2024.

\bibitem{Dissanayake94}
Mahesh Dissanayake and Nhan Phan-Thien.
\newblock Neural-network-based approximations for solving partial differential equations.
\newblock {\em Communications in Numerical Methods in Engineering}, 10:195--201, 1994.

\bibitem{E18}
Weinan E and Bing Yu.
\newblock The deep ritz method: A deep learning-based numerical algorithm for solving variational problems.
\newblock {\em Communications in Mathematics and Statistics}, 6:1--12, 2018.

\bibitem{gess-heydecker25}
Benjamin Gess and Daniel Heydecker.
\newblock The porous medium equation: Large deviations and gradient flow with degenerate and unbounded diffusion.
\newblock {\em Communications on Pure and Applied Mathematics}, 2025.

\bibitem{gess-heydecker-wu24}
Benjamin Gess, Daniel Heydecker, and Zhengyan Wu.
\newblock Landau-lifshitz-navier-stokes equations: Large deviations and relationship to the energy equality, 2024.

\bibitem{glorot10}
Xavier Glorot and Yoshua Bengio.
\newblock Understanding the difficulty of training deep feedforward neural networks.
\newblock In Yee~Whye Teh and Mike Titterington, editors, {\em Proceedings of the Thirteenth International Conference on Artificial Intelligence and Statistics}, volume~9 of {\em Proceedings of Machine Learning Research}, pages 249--256, Chia Laguna Resort, Sardinia, Italy, 13--15 May 2010. PMLR.

\bibitem{han18}
Jiequn Han, Arnulf Jentzen, and Weinan E.
\newblock Solving high-dimensional partial differential equations using deep learning.
\newblock {\em Proceedings of the National Academy of Sciences}, 115(34):8505--8510, August 2018.

\bibitem{hao23}
Zhongkai Hao, Jiachen Yao, Chang Su, Hang Su, Ziao Wang, Fanzhi Lu, Zeyu Xia, Yichi Zhang, Songming Liu, Lu~Lu, et~al.
\newblock Pinnacle: A comprehensive benchmark of physics-informed neural networks for solving pdes.
\newblock {\em arXiv preprint arXiv:2306.08827}, 2023.

\bibitem{hu24}
Ziqing Hu, Chun Liu, Yiwei Wang, and Zhiliang Xu.
\newblock Energetic variational neural network discretizations of gradient flows.
\newblock {\em SIAM Journal on Scientific Computing}, 46(4):A2528--A2556, 2024.

\bibitem{huang24}
Shenglin Huang, Zequn He, Nicolas Dirr, Johannes Zimmer, and Celia Reina.
\newblock Statistical-physics-informed neural networks (stat-pinns): A machine learning strategy for coarse-graining dissipative dynamics.
\newblock {\em Journal of the Mechanics and Physics of Solids}, page 105908, 2024.

\bibitem{hure20}
C{\^o}me Hur{\'e}, Huy{\^e}n Pham, and Xavier Warin.
\newblock Deep backward schemes for high-dimensional nonlinear pdes.
\newblock {\em Mathematics of Computation}, 89(324):1547--1579, 2020.
\newblock Published electronically: January 31, 2020.

\bibitem{jin-cai-li-karniadakis21}
Xiaowei Jin, Shengze Cai, Hui Li, and George~Em Karniadakis.
\newblock Nsfnets (navier-stokes flow nets): Physics-informed neural networks for the incompressible navier-stokes equations.
\newblock {\em Journal of Computational Physics}, 426:109951, 2021.

\bibitem{JKO98}
Richard Jordan, David Kinderlehrer, and Felix Otto.
\newblock The variational formulation of the fokker--planck equation.
\newblock {\em SIAM Journal on Mathematical Analysis}, 29(1):1--17, 1998.

\bibitem{kipnis99}
C.~Kipnis and C.~Landim.
\newblock {\em {Scaling Limits of Interacting Particle Systems}}.
\newblock Springer, 1999.

\bibitem{kipnis89}
C.~Kipnis, S.~Olla, and S.~R.~S. Varadhan.
\newblock Hydrodynamics and large deviation for simple exclusion processes.
\newblock {\em Communications on Pure and Applied Mathematics}, 42(2):115--137, 1989.

\bibitem{lagaris98}
Isaac Lagaris, Aristidis Likas, and Dimitrios Fotiadis.
\newblock Artificial neural networks for solving ordinary and partial differential equations.
\newblock {\em IEEE Transactions on Neural Networks}, 9:987--1000, 09 1998.

\bibitem{li21}
Zongyi Li, Nikola~Borislavov Kovachki, Kamyar Azizzadenesheli, Burigede Liu, Kaushik Bhattacharya, Andrew~M. Stuart, and Anima Anandkumar.
\newblock Fourier neural operator for parametric partial differential equations.
\newblock In {\em International Conference on Learning Representations (ICLR)}, 2021.
\newblock Accepted Jan 12 2021.

\bibitem{li24}
Zongyi Li, Hongkai Zheng, Nikola Kovachki, David Jin, Haoxuan Chen, Burigede Liu, Kamyar Azizzadenesheli, and Anima Anandkumar.
\newblock Physics-informed neural operator for learning partial differential equations.
\newblock {\em ACM / IMS J. Data Sci.}, 1(3), May 2024.

\bibitem{mishra-molinaro22}
Siddhartha Mishra and Roberto Molinaro.
\newblock {Estimates on the generalization error of physics-informed neural networks for approximating PDEs}.
\newblock {\em IMA Journal of Numerical Analysis}, 43(1):1--43, 01 2022.

\bibitem{mueller22}
Johannes M\"{u}ller and Marius Zeinhofer.
\newblock Error estimates for the deep ritz method with boundary penalty.
\newblock In Bin Dong, Qianxiao Li, Lei Wang, and Zhi-Qin~John Xu, editors, {\em Proceedings of Mathematical and Scientific Machine Learning}, volume 190 of {\em Proceedings of Machine Learning Research}, pages 215--230. PMLR, 15--17 Aug 2022.

\bibitem{Neelan24}
Arun~Govind Neelan, G.~Sai Krishna, and Vinoth Paramanantham.
\newblock Physics-informed neural networks and higher-order high-resolution methods for resolving discontinuities and shocks: A comprehensive study.
\newblock {\em Journal of Computational Science}, 83:102466, 2024.

\bibitem{park23}
Min~Sue Park, Cheolhyeong Kim, Hwijae Son, and Hyung~Ju Hwang.
\newblock The deep minimizing movement scheme.
\newblock {\em Journal of Computational Physics}, 494:112518, 2023.

\bibitem{peletier-redig-vafayi-14}
Mark~A. Peletier, Frank Redig, and Kiamars Vafayi.
\newblock Large deviations in stochastic heat-conduction processes provide a gradient-flow structure for heat conduction.
\newblock {\em Journal of Mathematical Physics}, 55(9):093301, 09 2014.

\bibitem{quastel98}
Jeremy Quastel and Horng-Tzer Yau.
\newblock Lattice gases, large deviations, and the incompressible navier-stokes equations.
\newblock {\em Annals of Mathematics}, 148(1):51--108, 1998.

\bibitem{raissi19}
M.~Raissi, P.~Perdikaris, and G.E. Karniadakis.
\newblock Physics-informed neural networks: A deep learning framework for solving forward and inverse problems involving nonlinear partial differential equations.
\newblock {\em Journal of Computational Physics}, 378:686--707, 2019.

\bibitem{ryck-mishra24}
Tim~De Ryck and Siddhartha Mishra.
\newblock Numerical analysis of physics-informed neural networks and related models in physics-informed machine learning.
\newblock {\em ArXiv}, abs/2402.10926, 2024.

\bibitem{DGM}
Justin Sirignano and Konstantinos Spiliopoulos.
\newblock Dgm: A deep learning algorithm for solving partial differential equations.
\newblock {\em Journal of Computational Physics}, 375:1339--1364, December 2018.

\bibitem{wang-chuwei22}
Chuwei Wang, Shanda Li, Di~He, and Liwei Wang.
\newblock Is $l^2$ physics-informed loss always suitable for training physics-informed neural network?, 2022.

\bibitem{Wang24}
Jie Wang, Xufeng Xiao, Xinlong Feng, and Hui Xu.
\newblock An improved physics-informed neural network with adaptive weighting and mixed differentiation for solving the incompressible navier–stokes equations.
\newblock {\em Nonlinear Dynamics}, 112:16113--16134, 2024.

\bibitem{yoshida17}
Hiroaki Yoshida.
\newblock A dissipation of relative entropy by diffusion flows.
\newblock {\em Entropy}, 19(1), 2017.

\bibitem{zang20}
Yaohua Zang, Gang Bao, Xiaojing Ye, and Haomin Zhou.
\newblock Weak adversarial networks for high-dimensional partial differential equations.
\newblock {\em Journal of Computational Physics}, 411:109409, 2020.

\bibitem{zhang-shin-karniadakis22}
Zhen Zhang, Yeonjong Shin, and George Em~Karniadakis.
\newblock Gfinns: Generic formalism informed neural networks for deterministic and stochastic dynamical systems.
\newblock {\em Philosophical Transactions of the Royal Society. A, Mathematical, Physical and Engineering Sciences}, 380(2229), 06 2022.

\bibitem{oettinger05}
Hans~Christian Öttinger.
\newblock {\em Beyond Equilibrium Thermodynamics}.
\newblock John Wiley \& Sons, Hoboken, NJ, 2005.

\end{thebibliography}

\end{document}